%% file: main_paper.tex
\documentclass{article}

\usepackage[final, eandd, nonatbib]{neurips_2026}

\usepackage[utf8]{inputenc} % allow utf-8 input
\usepackage[T1]{fontenc}    % use 8-bit T1 fonts
\usepackage{hyperref}       % hyperlinks
\usepackage{url}            % simple URL typesetting
\usepackage{booktabs}       % professional-quality tables
\usepackage{amsfonts}       % blackboard math symbols
\usepackage{nicefrac}       % compact symbols for 1/2, etc.
\usepackage{microtype}      % microtypography
\usepackage{xcolor}         % colors
\usepackage[maxbibnames=99]{biblatex}
\usepackage{adjustbox}
\usepackage[ruled,vlined]{algorithm2e}
\usepackage{amsmath}
\usepackage{amsthm}
\usepackage{subcaption} % Required for subfigures

\usepackage{makecell}
\usepackage{wrapfig}
\usepackage{cleveref}
\usepackage{enumitem}
\graphicspath{{./}{Formatting_Instructions_For_NeurIPS_2026/}}

\theoremstyle{plain}
\newtheorem{theorem}{Theorem}[section]
\newtheorem{proposition}[theorem]{Proposition}

\theoremstyle{definition}
\newtheorem{definition}[theorem]{Definition}
\newtheorem{example}[theorem]{Example}

\theoremstyle{remark}

\Crefname{obs}{Observation}{Observations}
\Crefname{obs}{Observation}{Observations}

\newcommand{\R}{\mathbb{R}}
\newcommand{\N}{\mathbb{N}}
\newcommand{\TRIP}{\textsc{TRIP}}
\newcommand{\GRIP}{\textsc{GRIP}}
\newcommand{\Prob}{\mathbb{P}}
\newcommand{\Hent}{\mathrm{H}}
\newcommand{\Nh}{\mathcal{N}}

\title{Get a GRIP, this will be a long TRIP: A Quantifiable Long-Range Framework for Verifying Over-squashing}

\author{%
  \textbf{Ferran Hernandez Caralt}\textsuperscript{1}\quad
  \textbf{Simon Heilig}\textsuperscript{2}\quad
  \textbf{Adrián Bazaga}\textsuperscript{3}\\
  \textbf{Asja Fischer}\textsuperscript{2}\quad
  \textbf{Moshe Eliasof}\textsuperscript{4}\quad
  \textbf{Pietro Li\`{o}}\textsuperscript{1}\\[1.2ex]
  \textsuperscript{1}Department of Computer Science and Technology, University of Cambridge, UK\\
  \textsuperscript{2}Faculty of Computer Science, Ruhr University Bochum, Germany\\
  \textsuperscript{3}Microsoft, London, UK\\
  \textsuperscript{4}Department of Applied Mathematics, University of Cambridge, UK\\[1.2ex]
  \texttt{fh455@cam.ac.uk}, \texttt{simon.heilig@rub.de}, \texttt{adrianbazaga@microsoft.com},\\
  \texttt{asja.fischer@rub.de}, \texttt{me532@cam.ac.uk}, \texttt{pl219@cam.ac.uk}
}

\begin{document}

\maketitle

\begin{abstract}
Empirical claims about the connection between over-squashing and long-range interactions in GNNs, can only be trusted if the benchmarks used to validate them genuinely require long-range interactions. The de-facto standard, the Long Range Graph Benchmark, has been repeatedly shown to be saturated by tuned short-range models, with existing synthetic alternatives being tied to specific topologies. As such, there is a lack of principled certificate of long-rangedness on \emph{arbitrary} graphs. This state reflects the absence of a precise characterization of long-ranged benchmarks. We address this fundamental gap by introducing four verifiable \emph{axioms}: \emph{Predictability}, \emph{Tightness}, \emph{Strictly $k$-Range}, and \emph{Topology-Invariance}, that any task claiming to test $k$-hop interactions must satisfy. We formally prove that violating any one of them admits failure modes that undermine conclusions drawn from the task. Based on these axioms, we introduce \TRIP{} (\emph{Truly Ranged Interactions Problem}) and its generalisation \GRIP{} (\emph{Generally Ranged Interactions Problem}), constructive procedures that turn \emph{any} graph into a provably long-ranged task by drawing features from \emph{stable} distributions. Moreover, by construction, \GRIP{} admits a closed-form, per-range Maximum-Likelihood oracle that yields the first \emph{a priori per-range} lower bound on test error available on any benchmark. Using our framework, we: (i) audit 4 common long-range benchmarks and identify their failures modes with respect to our axioms; (ii) on \TRIP-instantiated topologies, we find a popular notion of curvature is uncorrelated with GNN performance, supporting topological-vs-computational bottleneck distinction; and (iii) we show that a novel benchmark's over-squashing measures factors beyond pure long-rangedness. Code to use the framework and reproduce experiments is released \href{https://github.com/ferranhernandezc/graph-grip}{https://github.com/ferranhernandezc/graph-grip}.
\end{abstract}

\section{Introduction}
\label{sec:intro}

The dominant paradigm in graph representation learning, message-passing GNNs, propagates information by iteratively aggregating features from one-hop neighbors \cite{kipf2017semi,xupowerful}. This local inductive bias is highly effective in domains ranging from molecular property prediction to network optimisation \cite{shen2022graph,corso2024graph,wettewa2024graph,du2024densegnn}. However, it is also known to be limited when learning depends on the communication of  \emph{distant} nodes -- a failure mode linked to \emph{over-squashing} \cite{alon2020bottleneck,topping2022understanding,di2023over,black2023understanding}. A sizeable line of work has proposed architectures designed to mitigate over-squashing: graph rewiring \cite{topping2022understanding,gutteridge2023drew}, dynamical-systems-inspired layers \cite{gravinaanti,gravina2025oversquashing,heiligport,eliasof2025tango,bamberger2024bundle}, and graph transformers \cite{shirzad2023exphormer,toenshoff2021graph}. While these methodologies are different in their approach, the common thread is that the alleviation of over-squashing is validated by empirical gains on long-range benchmarks.

\paragraph{The Benchmarking Gap.} While many graph benchmarks exist, as discussed later, the core shortcoming is that there persists a lack of a principled way to verify that a real-world benchmark genuinely require long-range modeling. The Long Range Graph Benchmark (LRGB) \cite{dwivedi2022long}, one of the most common benchmarks, has now been thoroughly studied multiple times. In particular, it was shown that careful hyperparameter tuning and stronger final readout layers allow classical GNNs to match or exceed the performance of long-range tailored methods \cite{tonshoff2023did,luo2024classic,luo2025can}. Moreover, \cite{bambergermeasuring} introduced a \emph{post-hoc} Jacobian-based range measure, showing that the highest-performing models on LRGB use an effective range of at most $\approx 2$, i.e., not long-range interactions. Another avenue is synthetic alternatives, that rely on rings \cite{di2023over,black2023understanding}, periodic grids \cite{mathyslrim} and DAGs \cite{zhou2025glora}, and graph motifs \cite{miglior2025can}. These benchmarks improved the connection between over-squashing and long-range due to their construction, however, they also remain limited to specific, pre-defined and synthetic topologies. Therefore, the empirical evidence underpinning a substantial fraction of the over-squashing and long-ranged literature rests on benchmarks whose long-rangedness is unverified or limited to specific cases.

\paragraph{Our Approach.} To address the fundamental limitations discussed above, we seek to address the following question: given a graph task, can we \emph{a priori} certify that it requires range $k$, on \emph{any} graph? To fully answer this question, in this paper, we introduce the following:
\begin{itemize}[leftmargin=2em]
    \item \textbf{An axiomatic characterization of long-range tasks.} We state four verifiable axioms: \emph{Predictability}, \emph{Tightness}, \emph{Strictly $k$-Range}, \emph{Topology-Invariance}, and prove that any violation admits a benchmark in which sub-$k$-range models can solve a $k$-range task (\Cref{sec:axioms}).
    \item \textbf{A constructive framework that provably satisfies all four axioms.} \TRIP{} produces a $k$-range task on any graph by drawing features from a stable distribution; \GRIP{} composes \TRIP{} sub-problems for richer multi-scale tasks (\Cref{sec:tripgrip}).
    \item \textbf{A per-range MLE oracle.} For any \GRIP-generated task we derive a closed-form Maximum-Likelihood predictor conditioned on the $k$-hop neighborhood. To our knowledge this is the first \emph{a priori} lower bound on test error available for any long-range graph benchmark (\Cref{thm:mle}), thus enabling direct quantification of how much information GNNs can extract from certain ranges relative to the known optimal.
    \item \textbf{An audit of existing long-range benchmarks.} We identify the specific axiom that LRGB \cite{dwivedi2022long}, LRIM \cite{mathyslrim}, GLoRa \cite{zhou2025glora} and ECHO \cite{miglior2025can} violate (\Cref{sec:audit}).
    \item \textbf{Two empirical findings enabled by the framework.} On \TRIP-instantiated topologies, the curvature described in \cite{topping2022understanding} is uncorrelated with GNN error, supporting the topological-vs-computational bottleneck distinction of \cite{arnaiz2025oversmoothing} (\Cref{sb-sct:curvature}). On \TRIP-instantiated problems and the topologies introduced in \cite{miglior2025can}, the model developed in \cite{heiligport} dominates all Message Passing baselines -- despite not being dominant on the original tasks -- suggesting that the original tasks measure over-squashing influenced by factors other than long-rangedness (\Cref{sec:echo-trip}).
\end{itemize}

\section{Related Work}
\label{sec:related}

\paragraph{Empirical Long-Range Benchmarks.} LRGB \cite{dwivedi2022long} popularized the use of average graph diameter and domain knowledge as proxies for long-rangedness. Using it, more complex methods were proposed and validated on it \cite{eliasof2025tango, gravina2025oversquashing, arroyo2025vanishing, bamberger2024bundle, heiligport}. But these were followed by stronger baselines \cite{tonshoff2023did,luo2024classic,luo2025can}. Differently, the recent LRIM \cite{mathyslrim} replaces real data with the Ising model on periodic grids and provides a lower bound for short-range predictors; GLoRa \cite{zhou2025glora} introduces synthetic source-to-sink reachability tasks on DAGs; ECHO \cite{miglior2025can} considers synthetic shortest-path-based tasks, and DFT-based molecular tasks. Overall, all four benchmark suites address the over-squashing problem with different design choices, but none satisfy a uniform, consistent definition of long-rangedness across topologies.

\paragraph{Theoretical Measures of Over-squashing.} Sensitivity-based analyses \cite{topping2022understanding,di2023over,black2023understanding} bound the Jacobian $\partial h^{(L)}_v/\partial h^{(0)}_u$ (the hidden representation of node $v$ at layer $L$, $h^{(L)}_v$ w.r.t the one of node $u$ at layer $0$, $h^{(0)}_u$) in terms of model parameters and graph topology, identifying low Cheeger constant or negative balanced-Forman curvature as topological bottlenecks \cite{topping2022understanding}. \cite{arnaiz2025oversmoothing} draw a sharper distinction between \emph{topological bottlenecks} and \emph{computational bottlenecks} (vanishing gradients, ill-conditioned dynamics), and \cite{arroyo2025vanishing} argue the latter is the dominant driver in practice. \cite{bambergermeasuring, liang2025towards} introduce the Jacobian-based range measure of a function $\mathbf{F}$ on the features $\mathbf{X}$ of a graph $G = (V,E)$.
\begin{equation}
\hat{\rho}_u(\mathbf{F}) := \frac{\sum_{v \in V} \left| \tfrac{\partial (\mathbf{F}(\mathbf{X}))_u}{\partial \mathbf{x}_v} \right|\, d_G(u,v)}{\sum_{v \in V} \left| \tfrac{\partial (\mathbf{F}(\mathbf{X}))_u}{\partial \mathbf{x}_v} \right|},
\label{eq:bamberger-range}
\end{equation}
and use it to certify that high-performing LRGB models operate at range $\approx 2$. \emph{Our framework is complementary}: we provide \emph{a priori} task-design conditions guaranteeing minimum range, then use \Cref{eq:bamberger-range} to verify that GNNs trained on \TRIP{}/\GRIP{} indeed learn at the predicted range, and to expose a previously unnoticed averaging artifact of the measure, as discussed in \Cref{sec:tripgrip}. 

\paragraph{Architectural Mitigations of Over-squashing.} Approaches to alleviate over-squashing include curvature-aware rewiring \cite{topping2022understanding}, multi-hop and dynamic rewiring \cite{gutteridge2023drew}, dynamical-systems-inspired layers \cite{gravinaanti,gravina2025oversquashing,heiligport,eliasof2025tango,bamberger2024bundle}, and graph transformers \cite{shirzad2023exphormer,toenshoff2021graph}. A common and persistent limitation of this vast array of approaches is the available benchmarks, they lack a certification of whether a reported improvement reflects long-range modeling, normalization, read-out capacity \cite{tonshoff2023did}, or trainability \cite{arroyo2025vanishing,scholkemper2024residual}. The framework developed in our work, discussed in \Cref{sec:axioms,sec:tripgrip},  is designed to address this limitation by construction.

\section{Axiomatic Specification of Long-range Tasks}
\label{sec:axioms-section}
In this Section we present four axioms that characterize long-range tasks in graph learning. We begin by providing notation used throughout the paper, followed by presenting the axioms in \Cref{sec:axioms}, and we discuss their  properties and necessity in \Cref{sec:axioms_discussion}.

\paragraph{Notation.} Let $G = (V,E)$ denote a finite undirected graph. We use $x_u \in \R^{d_x}$ and $y_u \in \R^{d_y}$ for the (random) features and targets of node $u \in V$, with capitals $X_u, Y_u$ for the random variables. $\Hent$ denotes (conditional) entropy and $\mathrm{I}(\cdot;\cdot)$ mutual information. Given a node $u$, we write $G_u^k \subset G$ with nodes $V_u^k = \{v \in V : d_G(u,v) \le k\}$ for the $k$-hop closed neighborhood, and overload $G_u^k$ to denote the random variables including $\{X_v : v \in G_u^k\} \bigcup E_u^k$. % in particular, $Y_u \mid G_u^k$ should be read as $Y_u \mid \{X_v : v \in G_u^k\} \bigcup E_u^k$.
We write $\Nh_k(u) = \{v : d_G(u,v) = k\}$ for the $k$-hop \emph{boundary}. Finally, $\mathcal{P}_{<\infty}(\R)$ denotes finite multisets of $\R$.

\subsection{Characterizing Long-range Tasks}
\label{sec:axioms}

A task on graphs is a family of joint distributions $\{(X_u, Y_u)_{u \in V}\}$ indexed by graphs $G$. This statistical formulation is used to reflect that GNNs are statistical approximators, learning from data distributions. So, we argue that to benchmark the $k$-ranged capabilities of GNNs, tasks should require $k$-ranged information to be approximated, not solved \footnote{More details in Appendix \ref{app:axiom-theory}}. So, we say it is a \emph{$k$-range task} if it satisfies all of:

\begin{enumerate}[leftmargin=2em]
    \item \textbf{Predictability.} $Y_u = g(G_u^k)$ for some measurable $g$, for every $u$ and every $G$.
    \item \textbf{Tightness.} The family $\{Y_u : G,\, u \in V(G)\}$ is uniformly tight: for every $\delta > 0$ there exists $M < \infty$ such that $\Prob(\|Y_u\| > M) < \delta$ for all $G$ and all $u$.
    \item \textbf{Strictly $k$-Range.} $\mathrm{I}(Y_u\,;\,G_u^{k'}) = 0$ for every $k' < k$, and $\mathrm{I}(Y_u\,;\,G_u^k) > 0$. Equivalently, $\Hent(Y_u \mid G_u^{k'}) = \Hent(Y_u)$ for $k' < k$ and $\Hent(Y_u \mid G_u^k) < \Hent(Y_u)$.
    \item \textbf{Topology-Invariance.} For every $k' < k$ and every pair of graphs $G, \tilde G$ and nodes $u \in V(G), \tilde u \in V(\tilde G)$, $\Hent(Y_u \mid G_u^{k'}) = \Hent(Y_{\tilde u} \mid \tilde G_{\tilde u}^{k'})$.
\end{enumerate}

\paragraph{What each axiom rules out.} Axiom~1 is a sanity check; the proofs below extend straightforwardly to probabilistic $g$. Axiom~2 (\emph{Tightness}), the standard notion of distributional boundedness, rules out tasks whose targets blow up as the graph grows; we deliberately avoid the stronger ``almost-sure boundedness'' since that would exclude useful unbounded distributions like the Gaussian and the Cauchy. Axiom~3 (\emph{Strictly $k$-Range}) is the substantive long-range condition: \emph{no} sub-$k$-hop neighborhood carries any information about $Y_u$, so any model whose effective range is $< k$ is reduced to predicting $Y_u$ as a single random variable, so it cannot beat the constant-mean baseline. Axiom~4 (\emph{Topology-Invariance}) ensures that the conditional information content available below the target range is identical across nodes and graphs; without it, the same nominal task can be trivial on dense graphs and intractable on sparse ones, voiding cross-topology comparison.

\paragraph{What each axiom enables.} Together, Axioms~3 and 4 imply that $Y_u$ does not depend on $G_u^{k'}$ for $k' < k$, so the entropy $\Hent(Y_u)$ and the per-range MLE error are graph-independent constants determined only by the task-design parameters. This is what makes test errors on different graphs directly comparable, a property no existing long-range benchmark provides.

% \begin{remark}[Topological vs.\ effective range]
% \label{rem:eff-range}
% Strict $k$-rangedness is defined relative to the graph metric $d_G$ (here, shortest-path distance). It does \emph{not} preclude solving the task with fewer than $k$ \emph{layers} of a transformer with full attention, since a single attention head can reach the boundary of $\Nh_k(u)$ in one step. The axioms specify the topological range of the data; a model's \emph{effective} range relative to the bound \Cref{eq:bamberger-range} is a separate, model-dependent quantity \cite{bambergermeasuring, liang2025towards}.
% \end{remark}

\subsection{Theoretical Properties}
\label{sec:axioms_discussion}

In this Section we discuss the theoretical properties of the 4 proposed axioms in our framework. Notably, we show that not satisfying any one of Axioms~1--4 results in a benchmark whose performance gains can not be attributed to long-range modeling. We provide complete proofs to the theoretical results presented in this Section, in Appendix \ref{app:proofs} and more detailed explanations of the results in Appendix \ref{app:prop-theory}. 

\begin{proposition}[Tightness is necessary]
\label{prop:axiom2}
Fix $k > 1$ and let $X \sim \mathcal{N}(0,1)$. Let $\{G_n\}$ be a sequence of graphs in which every node $u$ satisfies $|\Nh_k(u)| \ge n$, and define $Y_u = \max_{v \in \Nh_k(u)} X_v$. The construction satisfies Axioms~1, 3, and 4, but the family $\{Y_u\}$ is not tight: for every $M \in \R$, $\Prob(Y_u > M) \to 1$ as $n \to \infty$.\footnote{If the underlying distribution is instead bounded, Axiom~4 fails; see Appendix \ref{app:proofs}.}
\end{proposition}

This Proposition implies that, if Axiom~3 is not satisfied graphs of increasing size may be built such that the norm of their predictions will become arbitrarily large. This forbids the study of GNN generalization over varying graph sizes.

\begin{proposition}[Strict $k$-Range is necessary]
\label{prop:axiom3}
Suppose a task satisfies Axioms~1, 2, 4 with target range $k$, but there exists $i < k$ with $\Hent(Y_u \mid G_u^i) < \Hent(Y_u \mid G_u^{i-1})$. Let $f^\star_i$ be the optimal range-$i$ predictor under any continuous error measure $\mathrm{err}$. Considering a finite set of known range-$k$ functions, and $f$ the one where the error measure is minimum, whenever $\mathrm{err}(f) > \mathrm{err}(f^\star_i)$, an improvement attained at range $i$ from $f^\star_{i-1}$ to $f^\star_i$ is reportable as an improvement at range $k$.
\end{proposition}

Notably, Proposition \ref{prop:axiom3} is reflected in real-world experiments --- it mirrors the documented shortcomings of LRGB: \cite{tonshoff2023did} and \cite{bambergermeasuring} show that range-2 models match or exceed putative long-range architectures on it.

\begin{proposition}[Topology-Invariance is necessary]
\label{prop:axiom4}
Fix $k > 1$. Sample $x_u \sim X$ i.i.d.\ from any distribution with finite variance, and let $y_u$ be the average of features at distance exactly $k$ from $u$. The task satisfies Axioms~1, 2, 3, but for every $\varepsilon > 0$ there exist graphs $G$ on which $\Hent(Y_u) < \varepsilon$.
\end{proposition}

The construction in Proposition \ref{prop:axiom4} works by taking $G$ with $|\Nh_k(u)|$ large, so that the average concentrates by the central limit theorem; the marginal entropy of $Y_u$ collapses, even though $Y_u$ formally depends on range-$k$ features. In practice, it implies that if Axiom~4 is not satisfied, the benchmark might be trivial for certain graphs, thus skewing research on long-range interactions models to not account for these graphs.

\section{Auditing Existing Long-range Benchmarks}
\label{sec:audit}

We now use Axioms~1--4 to audit the four most prominent long-range benchmarks, discussed below, and summarized in  \Cref{tab:audit}.
\paragraph{LRGB \cite{dwivedi2022long}.} Three independent re-evaluations \cite{tonshoff2023did,luo2024classic,luo2025can} show that classical short-range GNNs match or exceed self-described long-range methods after careful hyperparameter tuning, and \cite{bambergermeasuring} confirms via \Cref{eq:bamberger-range} that the strongest models operate at effective range $\approx 2$. This is a quantitative violation of Axiom~3 in the sense of Proposition \ref{prop:axiom3}: the optimal range-2 predictor is competitive with SOTA, so reported gains cannot be attributed to long-range modeling.

\paragraph{LRIM \cite{mathyslrim}.} The Ising-model construction on a periodic grid satisfies Axiom~3 by construction (a non-trivial lower bound for short-range predictors is known).  However, since the benchmark is restricted to a single topology family, Axiom~4 cannot be tested -- improvements may not transfer to general graphs.

% \vspace{-0.5cm}
\begin{wrapfigure}[13]{r}[0pt]{0.4\textwidth}
    \vspace{-0.455cm}
    \centering
    \includegraphics[width=0.4\textwidth]{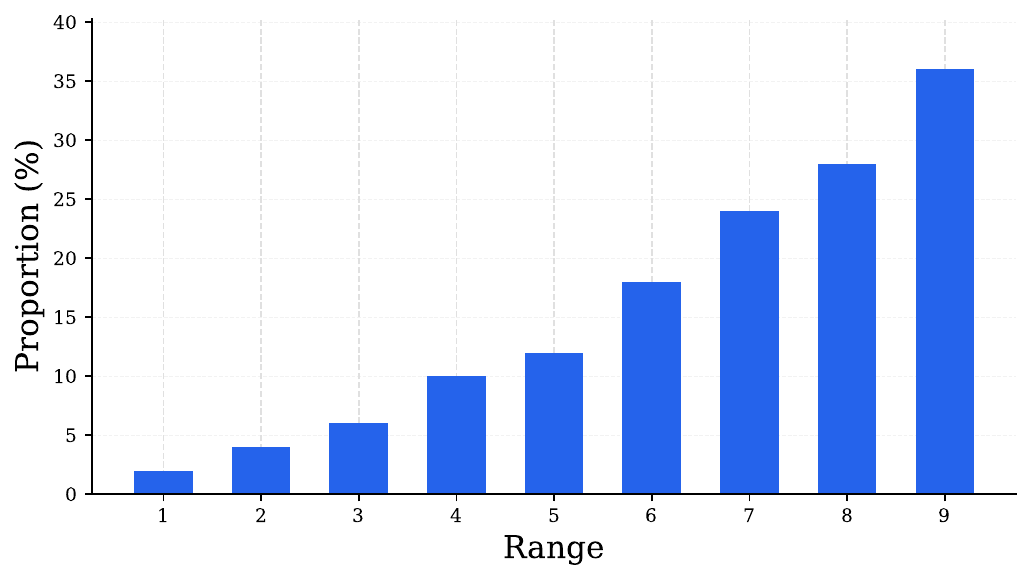}
    \vspace{-0.455cm}
    \caption{Percentage of GLoRa graphs (alleged range $\in [10,15]$, parameters as in \cite{zhou2025glora}) classifiable using only information within range $\le k$ for $k = 1, \ldots, 9$.}
    \label{fig:glora_counter}
\end{wrapfigure}
\paragraph{GLoRa \cite{zhou2025glora}.} A purely synthetic benchmark on directed acyclic graphs with a designated source and sink. Node features come in three classes (source, filler, hole); the task is to detect whether a path exists from source to sink avoiding hole nodes. The work claims that solving this requires range $k$ equal to the source-hole distance. We point out two issues. First, as a synthetic DAG benchmark restricted to one topology family, the LRIM caveat applies: Axiom~4 cannot be tested. Second, GLoRa fails Axiom~3, because its construction does not enforce that every instance \emph{requires} long-range information: if for example all neighbors of the source are hole nodes, a range-1 model classifies the instance correctly. Figure \ref{fig:glora_counter} quantifies this. On GLoRa graphs generated at requested ranges $10$--$15$, $\approx 35\%$ of instances are correctly classifiable using only information at hops $\le 9$. A range-$9$ classifier therefore achieves at least $\approx 70\%$ accuracy on GLoRa-15 (correct on the $35\%$ short-resolvable instances; majority-class on the rest), without any range-$10$--$15$ information.

% \vspace{-0.51cm}
\begin{wraptable}{l}{0.6\linewidth}
\vspace{-0.4cm}
\centering
\tiny
\caption{Axioms satisfied by existing long-range benchmarks. ``\textbf{?}'' marks an axiom that cannot be tested because the benchmark is restricted to a single topology family. ``$Q$-violation'' under Axiom~3 indicates a quantitative violation supported by post-hoc measurement.}
% \vspace{0.1cm}
\label{tab:audit}
\begin{tabular}{lccccl}
\toprule
Benchmark & Ax.~1 & Ax.~2 & Ax.~3 & Ax.~4 & Evidence \\
\midrule
LRGB \cite{dwivedi2022long} & \checkmark & \checkmark & $Q$v & \checkmark & \cite{tonshoff2023did,luo2024classic,luo2025can,bambergermeasuring} \\
LRIM \cite{mathyslrim} & \checkmark & \checkmark & \checkmark & \textbf{?} & Periodic grid only \\
GLoRa \cite{zhou2025glora} & \checkmark & \checkmark & $\times$ & \textbf{?} & DAG only; \Cref{fig:glora_counter} \\
ECHO-\texttt{sssp} \cite{miglior2025can} & \checkmark & \checkmark & $Q$v & \textbf{?} & \Cref{fig:echo_range} \\
ECHO-\texttt{ecc.} \cite{miglior2025can} & \checkmark & \checkmark & $Q$v & \textbf{?} & \Cref{fig:echo_range} \\
ECHO-\texttt{charge} \cite{miglior2025can} & \checkmark & \checkmark & $Q$v & \textbf{?} & \Cref{fig:echo_range} \\
\midrule
\TRIP/\GRIP{} (this work) & \checkmark & \checkmark & \checkmark & \checkmark & \Cref{thm:trip} \\
\bottomrule
\end{tabular}
\vspace{-0.5cm}
\end{wraptable}
\paragraph{ECHO \cite{miglior2025can}.} Three synthetic shortest-path-style sub-tasks (\texttt{sssp}, \texttt{eccentricity}, \texttt{diameter}) on randomized instances of base topology families, plus two real-world DFT-based molecular tasks (\texttt{charge}, \texttt{energy}). We note that, eccentricity, for instance, is the longest shortest path from a node, which is plausibly correlated with the degrees of the immediate neighbors of that node -- but the achievable accuracy from such correlations alone is not characterized in \cite{miglior2025can}. To probe this empirically, we trained ECHO baselines with the recommended hyperparameters in \cite{miglior2025can}, and computed their post-hoc range \Cref{eq:bamberger-range}. \Cref{fig:echo_range} reports the relationship between measured range and test error. Our findings show that only on \texttt{sssp}, a higher measured range correspond to lower error. However, on \texttt{charge} and \texttt{eccentricity} we observe no correlation. Moreover, the trained models reach a measured range of only $\le 4$. Hence, range-$> 4$ behavior cannot be certified by evaluating with ECHO benchmarks. Moreover, we note that each ECHO sub-task considers a distinct topology family, so Axiom~4 cannot be tested as well.

\begin{figure}
    \centering
    \hspace{-0.2cm}
    \includegraphics[width=0.95\linewidth, height=4cm]{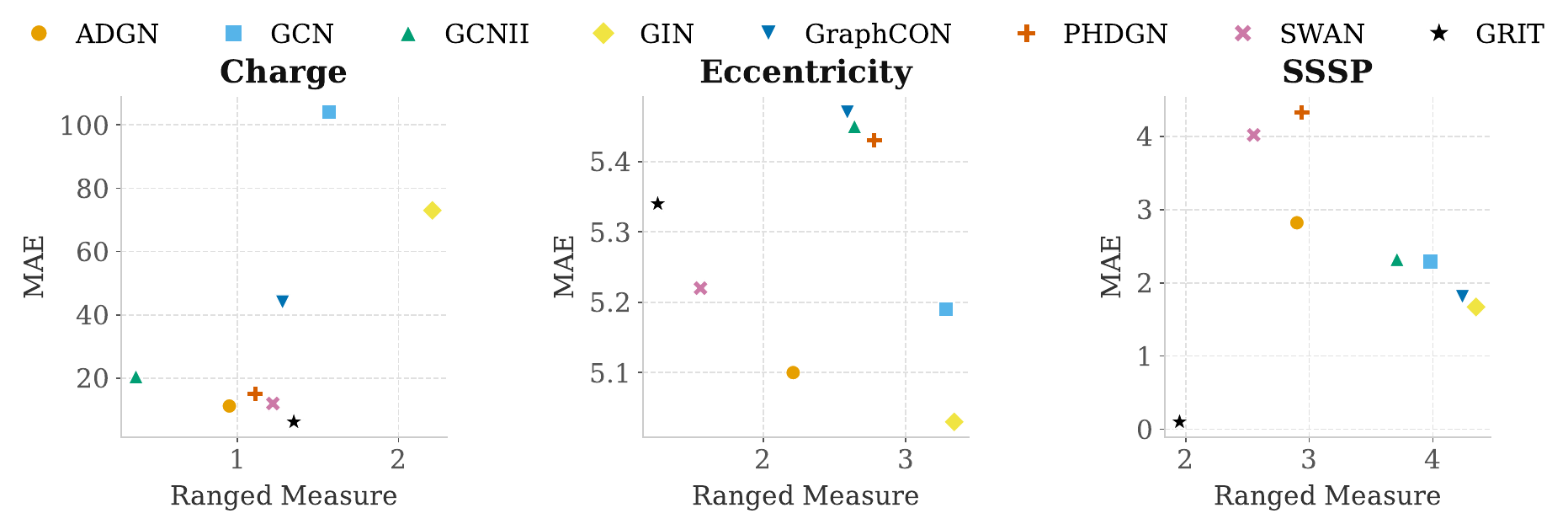} 
    \caption{Test error vs.\ post-hoc range \Cref{eq:bamberger-range} for ECHO baselines trained with the hyperparameters of \cite{miglior2025can}. On none of these node regression tasks does higher measured range correspond to lower error. 
    }
    \label{fig:echo_range}
S
    \vspace{-0.5cm}
\end{figure}

In conclusion, the long-ranged nature of benchmarks LRGB \cite{dwivedi2022long}, GLoRa \cite{zhou2025glora} and ECHO \cite{miglior2025can} is unknown. LRIM \cite{mathyslrim} provides solid theoretical guarantees of its long-rangedness while using a real-world physics problem, making it a strong benchmark to assess GNNs that focus on long-range interactions modelling. It's only limitation is that the graph topology is fixed. Consequently, developments tested exclusively on it might not generalize to other topologies. This analysis motivates the need of a general framework to stress test GNNs under arbitrary ranges and topologies.

\section{Constructing Certifiable Long-range Benchmarks}
\label{sec:tripgrip}

The four axioms in \Cref{sec:axioms-section} can be satisfied on any graph or collection of graphs, provided we choose the appropiate $g$ (from Axiom~1) and sample the features from the distributions with the right properties to satisfy the other axioms. More specifically, we will use a concept, called \emph{strict stability of distributions under aggregation}. In what follows, we formalize this concept as it is a necessary to give the definition of the task that may satisfy axioms 1 through 4. %-- is enough to satisfy all of them on \emph{any} graph or collection of graphs. We make this precise.

\begin{definition}[$f$-stable distribution]\footnote{Definition \ref{def:fstable} mirrors the classical notion of \emph{strict stability under} $f$, which generalizes the well-known stability of Gaussian, Cauchy, and L\'evy distributions under summation \cite{satoLevyProcesses}}
\label{def:fstable}
For $f \colon \mathcal{P}_{<\infty}(\R) \to \R$, a real random variable $X$ is \emph{$f$-stable} if there exists a random variable $Y$ such that, for every $n \ge 1$ and every i.i.d.\ sample $X_1, \ldots, X_n \sim X$, $f(\{X_1, \ldots, X_n\}) \overset{d}{=} Y$.
\end{definition}

\begin{example}
    We give three examples of such distribution-function pairs covering the normal, heavy-tailed, and extreme-value regimes: (i) $X \sim \mathcal{N}(0, \sigma^2)$ with $f(\{x_1, \ldots, x_n\}) = (x_1 + \cdots + x_n)/\sqrt{n}$, (ii) $X \sim \mathrm{Cauchy}(x_0, \gamma)$ with $f(\{x_1, \ldots, x_n\}) = (x_1 + \cdots + x_n)/n$, and (iii) $X \sim \mathrm{Gumbel}(\mu, \beta)$ with $f(\{x_1, \ldots, x_n\}) = \max_i x_i - \beta \ln n$.
\end{example}

\begin{proposition}[Closure properties]
\label{prop:closure}
Let $X$ be $f$-stable and $\tilde X$ be $\tilde f$-stable. Then: (i) for any measurable $g \colon \R \to \R$, $X$ is $(g \circ f)$-stable, and (ii) The pair $(X, \tilde X)$ is $F$-stable.
\end{proposition}

\subsection{Truly Ranged Interactions Problem (TRIP)}

To make sure Axiom~1 is satisfied, it is sufficient to deterministically calculate $Y_u$. To define a task that satisfies Axiom~3, we sample features i.i.d.\ at each node, and define the target as a function of the features at distance \emph{exactly} $k$. To satisfy Axioms~3,4, we use an aggregator $f$ for which the feature distribution is $f$-stable: the aggregate then has the \emph{same} distribution regardless of how many neighbors reside at the boundary.

\begin{definition}[\TRIP]
\label{def:trip}
A \emph{Truly Ranged Interactions Problem} (\TRIP) at range $k$ is parameterized by an $f$-stable pair $(X, f)$ and consists, for any graph $G = (V,E)$, of: (i) sampling features $x_u \sim X$ i.i.d.\ for every $u \in V$, and (ii) setting targets $y_u = f(\{x_v : v \in \Nh_k(u)\})$. The task is to predict $y_u$ for every $u \in V$ given $G$ and $\{x_u\}_{u \in V}$.
\end{definition}

\begin{theorem}[\TRIP{} satisfies all four axioms]
\label{thm:trip}
For any graph $G$ and any $f$-stable pair $(X, f)$, the \TRIP{} task at range $k$ satisfies Axioms~1--4. In particular, the marginal $Y_u$ and the entropy $\Hent(Y_u)$ are graph-independent constants determined only by $(X, f)$.
\end{theorem}

\Cref{thm:trip} implies that, we can apply TRIP to any kind of graph and the statistical properties of the features $X$ and the targets $Y$ will remain consistent. Thus enabling direct comparisons of over-squashing under long-rangedness accross different graphs. Furthermore, in \cite{abboud2020surprising} it was shown that a  message-passing neural network (MPNN) equipped with random node features are universal approximators of permutation-invariant functions on graphs of bounded size; Any \TRIP{}-generated benchmark can therefore be learned by MPNN equipped with random node features. The empirical question we address in \Cref{sec:cycle-val} is: \emph{at which range a given architecture fails to learn them?}

\subsection{Generalized Ranged Interactions Problem and the per-range MLE oracle}
\label{sec:grip}

To express tasks combining information from multiple ranges and feature dimensions, we generalize \TRIP{} to a multivariate target whose complexity is controlled by the set of involved ranges and the choice of $f$-stable distributions.

\begin{definition}[\GRIP]
\label{def:grip}
Let $\{(X_i, f_i)\}_{i=1}^{d_x}$ be $d_x$ pairs of $f_i$-stable distributions and let $K_i \subseteq \N_{\ge 1}$ be a finite set of ranges per feature. For each node $u$ sample $x_u \in \R^{d_x}$ with $(x_u)_i \sim X_i$ independently. For each $(i, k) \in \{1, \dots, d_x\} \times K_i$, define $\hat y_u^{(i, k)} = f_i\big(\{(x_v)_i : v \in \Nh_k(u)\}\big)$. Stack into $\hat y_u \in \R^{\sum_i |K_i|}$ and apply a measurable read-out $h \colon \R^{\sum_i |K_i|} \to \R^{d_y}$ to obtain $y_u = h(\hat y_u)$.
\end{definition}

Following the logic in \Cref{thm:trip}, \GRIP{} satisfies Axioms~1--4 with target range $k = \max\bigcup_i K_i$. Notably, \GRIP{} admits a closed-form per-range Maximum-Likelihood predictor, the only \emph{a priori} lower bound on test error available for any range on any graph.

\begin{theorem}[Per-range MLE oracle]
\label{thm:mle}
Fix a \GRIP{} task with parameters $\{(X_i, f_i, K_i)\}_i$ and read-out $h$, and let $k^\star \ge 0$ be a target observation range. Let $\hat X_i$ be the MLE realisation of $X_i$. The Maximum-Likelihood predictor of $Y_u$ given $G_u^{k^\star}$ is $h(\tilde y_u)$, where for each $(i,k)$,
\[
\tilde y_u^{(i,k)} =
\begin{cases}
f_i\big(\{(x_v)_i : v \in \Nh_k(u)\}\big) & \text{if } k \le k^\star,\\
f_i\big(\{\hat X_i\}\big) & \text{if } k > k^\star.
\end{cases}
\]
The expected error of this predictor is a graph-independent function of $(\{(X_i, f_i, K_i)\}_i, h, k^\star)$ alone, and provides a lower bound for any model with effective range $\le k^\star$.
\end{theorem}

\Cref{thm:mle} is a result of the invariance principle for MLEs \cite[\S 7.2]{casella2024statistical}. \Cref{thm:mle} allows addressing a question that existing long-range benchmark lack: \emph{at what range does a model begins to outperform the optimal range-$(k^\star{-}1)$ predictor?} For example: for $X \sim \mathcal{N}(0,1)$ and $f(\{x_1, \ldots, x_n\}) = (x_1 + \cdots + x_n)/\sqrt{n}$, the MLE of $X$ from a single observation is $\hat X = 0$, hence $f(\{\hat X\}) = 0$.

\subsection{Cycle-graph Validation}
\label{sec:cycle-val}

\setcounter{figure}{3}
\begin{wrapfigure}[19]{r}[0pt]{0.4\textwidth}
    \centering
    \vspace{-0.55cm}
    \includegraphics[width=\linewidth]{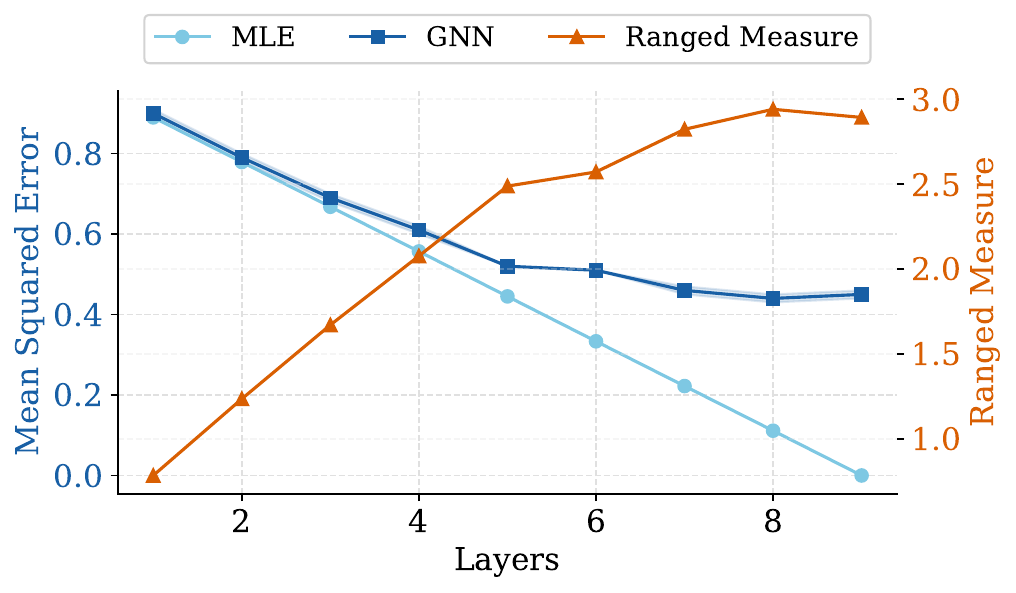}
    \caption{\GRIP{} on a 22-cycle reveals an averaging bias in \Cref{eq:bamberger-range}. Test MSE (solid; left $y$-axis) and measured range \Cref{eq:bamberger-range} (dashed; right $y$-axis) for GCNs trained on a $9$-dimensional \GRIP{} (range set $\{1, \ldots, 9\}$, Gaussian-sum aggregation, cubic-root read-out) as a function of model depth ($x$-axis). The MLE oracle from \Cref{thm:mle} is shown as a dotted line; the GCN matches it for ranges $\le 5$ and saturates afterwards.
    }
    \label{fig:grip_cycles}
\end{wrapfigure}
To verify that \TRIP{} behaves as predicted, we trained GCN with one residual connection per layer (with complete experimental details provided in Appendix \ref{app:hyper-param}) on standard \TRIP{} on cycles of length $22$ for ranges $k \in \{1, \ldots, 9\}$, with model depths $\{k{-}1, k, k{+}1\}$, repeated for the three $f$-stable distributions shown in \Cref{sec:tripgrip}. In \Cref{fig:cycle_plots}  we show that (i) the $(k{-}1)$-layer GCN consistently struggles learning the task (i.e., it yields high MSE; the learned range is $\approx 0$, as predicted by Axiom~3); (ii) the $k$-layer GCN learns the task at low $k$ and its measured range \Cref{eq:bamberger-range} tracks $k$, with performance degrading for $k \in \{7,8\}$ across all three distributions, consistent with prior over-squashing observations on cycles; and (iii) the $(k{+}1)$-layer GCN does \emph{not} help past $k=7$, ruling out under-reaching \cite{gabrielsson2023rewiring}  as the failure mode \cite{arroyo2025vanishing}.

\setcounter{figure}{2}
\begin{figure}
    \centering
    \includegraphics[width=0.95\linewidth]{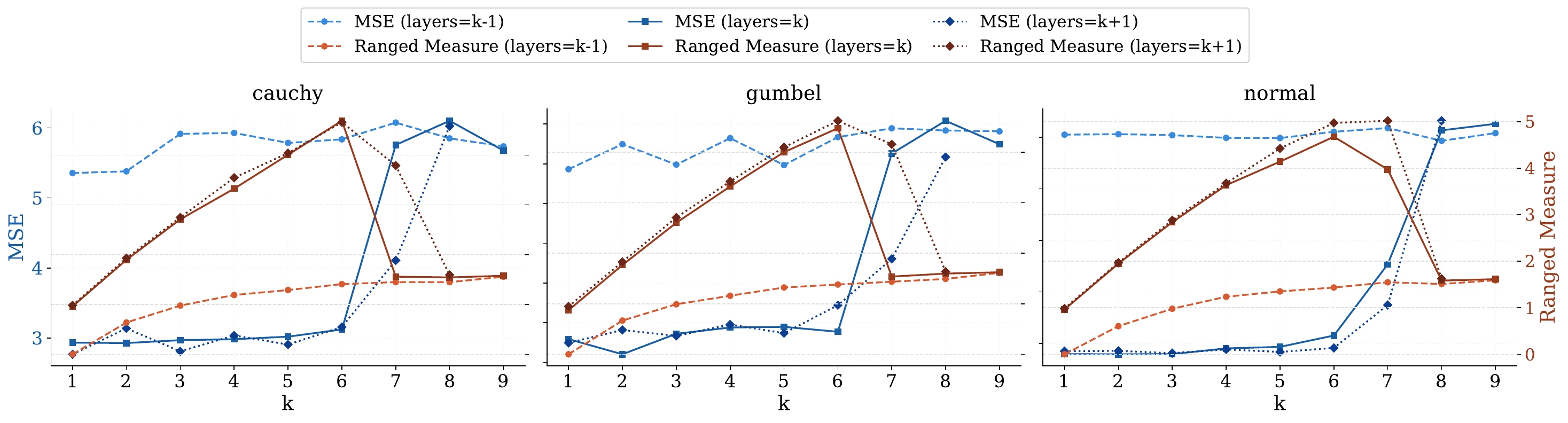}
    \caption{\TRIP{} on a 22-cycle isolates an over-squashing transition at $k \in \{7,8\}$ that is not under-reaching. Test MSE (left $y$-axes) and post-hoc measured range \Cref{eq:bamberger-range} (right $y$-axes) for residual GCNs of depth $\{k{-}1, k, k{+}1\}$, $k \in \{1, \ldots, 9\}$, on the three $f$-stable distributions of \Cref{def:fstable}. The $(k{-}1)$-layer GCN uniformly fails; the $k$- and $(k{+}1)$-layer GCNs both succeed at low $k$ and break down together at $k \in \{7,8\}$, ruling out under-reaching.}
    \label{fig:cycle_plots}
    \vspace{-0.2cm}
\end{figure}

\paragraph{The averaging artifact of \Cref{eq:bamberger-range}.} For \GRIP{}, we train a GCN \cite{kipf2017semi} on a 22-cycle with $9$-dimensional features, each aggregated at ranges $1, \ldots, 9$ via the Gaussian sum and combined via $h(z) = (\sum_{k=1}^9 z_k)^{1/3}$. \Cref{fig:grip_cycles} shows that (i) on ranges $1$--$5$ the GCN attains the MLE bound from \Cref{thm:mle}; (ii) the model cannot extract information from ranges $6$--$9$ even though messages reach those nodes; and (iii) the post-hoc measure \Cref{eq:bamberger-range} reports an effective range of $\approx 3$, despite the model demonstrably using information at ranges $4$ and $5$. Importantly, we note that this discrepancy is not measurement error: \Cref{eq:bamberger-range} is an \emph{influence-weighted} average of ranges, hence, it lacks information on the distribution of the weights. For example, a MPNN that needs information from ranges 1 and 9 equally, will have measured range 5, underestimating the range.

\section{Empirical Findings}
\label{sec:showcases}

We utilize the TRIP framework to address two research questions where existing benchmarks have so far given ambiguous answers: (i) does the balanced-Forman curvature in \cite{topping2022understanding} predict GNN error on truly long-range tasks? (ii) is performance on the ECHO \cite{miglior2025can} benchmark solely explained by long-rangedness? Throughout, ``standard \TRIP{}'' refers to \TRIP{} with $X \sim \mathcal{N}(0,1)$ and $f(\{x_1, \ldots, x_n\}) = (x_1+\cdots+x_n)/\sqrt{n}$.

\subsection{Curvature is uncorrelated with GNN error on \TRIP{}}
\label{sb-sct:curvature}

    The \emph{balanced-Forman curvature} was introduced in \cite{topping2022understanding} as an efficient combinatorial measure of bottlenecks and proposed Stochastic Discrete Ricci Flow (SDRF), a rewiring technique that surgically targets negatively-curved edges. In recent works, the distinction between \emph{topological} bottlenecks (low Cheeger constant, negative curvature) and \emph{computational} bottlenecks (model dynamics, vanishing gradients) has been introduced \cite{arnaiz2025oversmoothing}. Posing the conjecture that curvature-based rewiring addresses only the former. The hypothesis we test is as follows: \emph{if curvature-based rewiring resolves the long-range component of over-squashing, then on a provably $k$-range task, (a) the curvature of the input graph should predict GNN error, and (b) SDRF rewiring should reduce that error.} We test both predictions under standard \TRIP{}, which by Axiom~4 isolates the long-range component from confounders.

\paragraph{Setup and Findings.} For $k \in \{1, \ldots, 7\}$ we sample $100$ random graphs of $\approx 1000$ nodes (expected-degree generator); whenever a node has empty $\Nh_k(u)$ its target is set to $0$. We compare a $k$-layer residual GCN under (i) no rewiring (\textsc{Exp\_Deg}) and (ii) SDRF \cite{topping2022understanding}, with three seeds each.\footnote{Cross-graph error comparison is not fair in this case because $|\Nh_k(u)|$ is $0$ for $\sim 70\%$ ($k{=}6$) and $\sim 90\%$ ($k{=}7$) of the nodes of SDRF-rewired graphs. A More detailed discussion on these cases can be found in Appendix \ref{app:graph-comparison}} \Cref{fig:curv_errs} shows that SDRF and the original graph perform nearly identically for $k \le 5$, with the only divergence appearing at $k \in \{6,7\}$ -- precisely where SDRF has made struggles with long-range interaction modeling. \Cref{fig:curv_corrs} reports the Spearman correlation, \emph{within each range}, between minimum edge curvature and post-training MSE: $|\rho| < 0.4$ throughout, we see no significant trend. Therefore, we conclude that per-graph curvature is \emph{not predictive} of GNN error on \TRIP{}, shedding light on  how the hypothesis above actually manifest in practice. This provides, for the first time, direct empirical evidence of the topological-vs-computational bottleneck distinction \cite{arnaiz2025oversmoothing,arroyo2025vanishing}.

\setcounter{figure}{4}
\begin{figure}[htbp]
    \centering
    \begin{subfigure}[c]{0.4\textwidth}
        \centering
        \includegraphics[width=\linewidth]{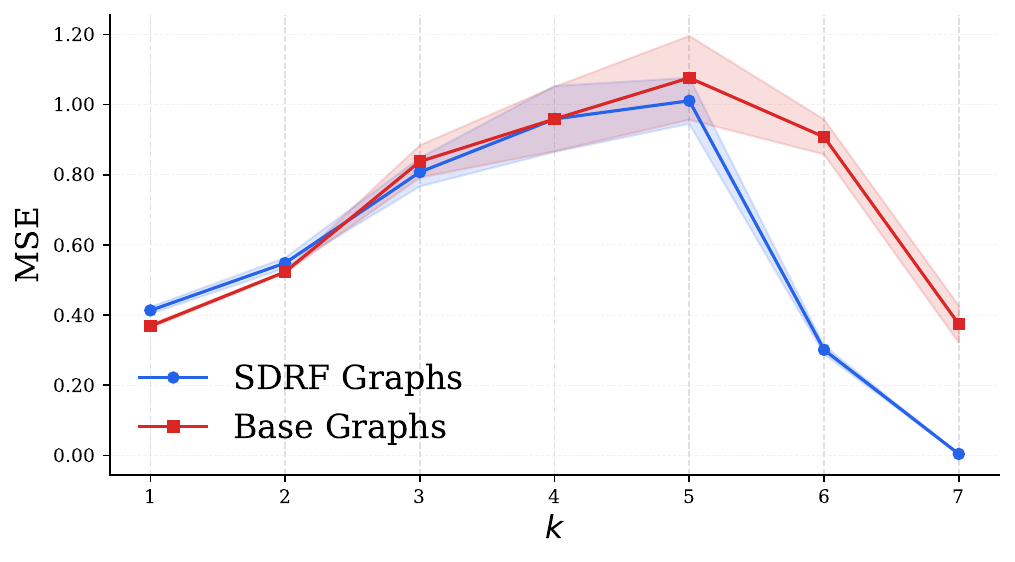}
        \caption{Test MSE per range, original vs.\ SDRF-rewired graphs.}
        \label{fig:curv_errs}
    \end{subfigure}
    \hfill
    \begin{subfigure}[c]{0.4\textwidth}
        \centering
        \vspace{0.2cm}
        \includegraphics[width=\linewidth]{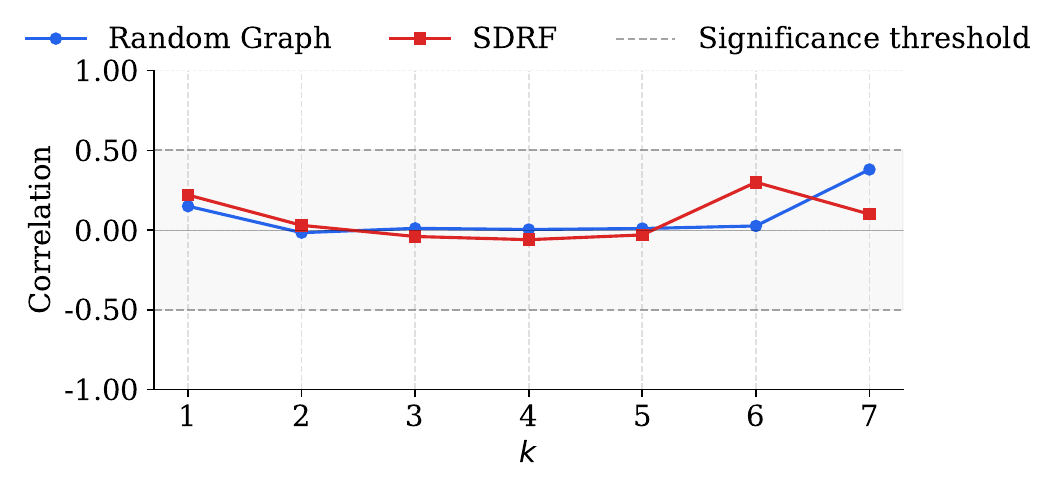}
        \caption{Spearman $\rho$ between minimum edge curvature and post-training MSE, per range.}
        \label{fig:curv_corrs}
    \end{subfigure}
    \caption{We show Balanced-Forman curvature is uncorrelated with GNN error on \TRIP{}. Left: SDRF and original graphs are indistinguishable in MSE for $k \le 5$; the divergence at $k \in \{6,7\}$ is an artifact of SDRF eliminating the boundary $\Nh_k(u)$ in most nodes (see footnote). Right: per-range Spearman correlation between minimum edge curvature and MSE stays in $|\rho| < 0.4$ throughout, opposing the prediction that curvature should track GNN error on a long-range task. Standard deviation is computed across $3$ seeds.}
    \label{fig:curvature}
\end{figure}
\vspace{-0.2cm}
\subsection{ECHO under \TRIP: over-squashing beyond long-rangedness}
\label{sec:echo-trip}

Our analysis of ECHO in \Cref{sec:audit}  shows that on \texttt{charge} and \texttt{eccentricity}, post-hoc measured range is uncorrelated with test error. Thus, it is unclear whether the tasks are guaranteed to require long-range interactions. We now utilize \TRIP{} to make this a controlled test. In particular, we ask the following question: \emph{on the same topologies, under a task that is provably $k$-range, do the same architectures rank the same way as on the original ECHO tasks?} A positive answer would imply that the bottleneck to improving performance is, indeed, the long-rangedness of the task.

\paragraph{Setup and Findings.} We instantiate standard \TRIP{} on the two ECHO topology classes: the molecular graphs of \texttt{charge} and the perturbed synthetic graphs of \texttt{eccentricity} (line, ladder, grid-like, tree, caterpillar, lobster). Because \TRIP{} satisfies Axiom~4, MSE values are directly comparable across these two classes. We train all message-passing baselines released with ECHO\footnote{DRew \cite{gutteridge2023drew} and GPS \cite{rampavsek2022recipe} are excluded since its rewiring operate on topologies in which all nodes reside at distance 1.} for ranges $k \in \{2, 4, 6, 8\}$ with the same hyperparameter sweep protocol as \cite{miglior2025can}; details are provided in Appendix \ref{app:hyper-param}, and results are reported in \Cref{fig:echo_trip}. Two observations stand out.\ First, the synthetic topologies are systematically harder than the molecular ones, validating their intended design as networks with bottlenecks. Second, PHDGN \cite{heiligport} outperforms all baselines on the \TRIP-instantiated synthetic topologies, a gap that is absent on the original ECHO tasks. By $k=4$, the classical baselines (GCN \cite{kipf2017semi}, GIN \cite{xupowerful}, GCNII \cite{chen2020simple}, GraphCON \cite{rusch2022graph}, ADGN \cite{gravinaanti}, GRIT \cite{ma2023graph}) saturate at $\mathrm{MSE} = 1.0$ (the constant-mean baseline), and PHDGN and SWAN \cite{gravina2025oversquashing} remain meaningful learners.

\paragraph{Implications.} If long-rangedness were the dominant bottleneck on ECHO's \texttt{eccentricity}, \texttt{sssp}, and \texttt{charge}, the PHDGN-vs-baseline gap we observe under \TRIP{} should also appear there. It does not. The gap on ECHO is therefore explained, at least in part, by factors other than pure long-rangedness -- candidate explanations include computational bottlenecks/vanishing gradients \cite{arnaiz2025oversmoothing,arroyo2025vanishing}, mixing capacity \cite{di2023over}, or task-specific inductive biases. Our framework does not adjudicate between these, but it does isolate the long-range component cleanly enough to rule out long-rangedness alone as the driver \footnote{One could argue that assuming i.i.d features also enforces low feature homophily, for a more detailed discussion on this idea see Appendix \ref{app:homo}}. More broadly, this experiment illustrates a workflow the framework enables: \emph{instantiate \TRIP{} on a benchmark's topologies; compare relative model rankings under \TRIP{} against the original task; the discrepancy quantifies how much of the original task gap is not long-range-driven.}

\section{Conclusion}
\label{sec:conclusion}

We demonstrated that the absence of a precise specification of what makes a long-range benchmark long-range has stalled progress on over-squashing. To address this gap,  we introduce a certifiable framework: four verifiable axioms (\Cref{sec:axioms}), a constructive procedure -- \TRIP{} and \GRIP{} -- that satisfies them on any graph topology (\Cref{sec:tripgrip}), and a closed-form per-range MLE oracle. To the best of our knowledge, this is the first \emph{a priori} lower bound on test error available for any long-range graph benchmark (\Cref{thm:mle}). Materializing our framework into real-world benchmarks has produced two non-trivial findings: (i) balanced-Forman curvature is uncorrelated with GNN error on truly long-range tasks (\Cref{sb-sct:curvature}); and (ii) the model-ranking gap on ECHO is not explained by long-rangedness alone (\Cref{sec:echo-trip}). We also explicitly surfaced how averaging artifact of the post-hoc range measure of \cite{bambergermeasuring, liang2025towards} can be misleading (\Cref{sec:tripgrip}). As future work, kernel-based extensions of $f$-stability for correlated features and relaxations of Axiom~1 to probabilistic $g$ to model noise could close the gap to real-world data. We envision our framework propelling the community toward long-range comparisons with precisely understood scope.

\begin{ack}
Ferran Hernandez Caralt acknowledges that the project that gave rise to these results received the support of a fellowship from “la Caixa” Foundation (ID 100010434). The fellowship code is LCF/BQ/PFA25/11000012. Simon Heilig, Asja Fischer acknowledge funding by Deutsche Forschungsgemeinschaft (DFG, German Research Foundation) via the TRR 391 Spatio-temporal Statistics for the Transition of Energy and Transport (project 520388526)
\end{ack}

\begin{figure}
    \centering
    \includegraphics[width=0.9\linewidth]{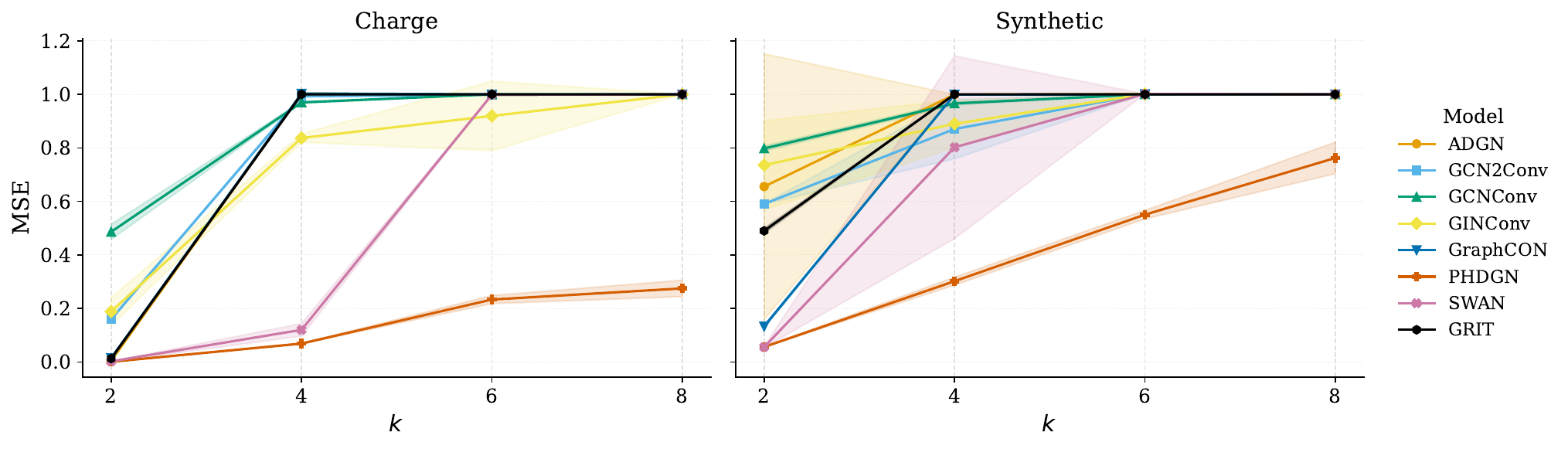}
    \caption{Test MSE on \TRIP-instantiated ECHO topologies (\texttt{charge}, \texttt{synth}) for $k \in \{2,4,6,8\}$. PHDGN \cite{heiligport} dominates all baselines, with a gap that is absent on the original ECHO tasks of \cite{miglior2025can}. By $k=4$ the classical baselines (GCN \cite{kipf2017semi}, GIN \cite{xupowerful}, GCNII \cite{chen2020simple}, GraphCON \cite{rusch2022graph}, ADGN \cite{gravinaanti}, GRIT \cite{ma2023graph}) saturate at $\mathrm{MSE}=1.0$ (constant-mean baseline); only PHDGN and SWAN \cite{gravina2025oversquashing} remain meaningful learners. Standard deviation is computed over $4$ seeds \cite{miglior2025can}.}
    \label{fig:echo_trip}
\end{figure}

\printbibliography

\appendix

\include{appendices}

%%%%%%%%%%%%%%%%%%%%%%%%%%%%%%%%%%%%%%%%%%%%%%%%%%%%%%%%%%%%

\newpage
\input{checklist.tex}

\end{document}

%% file: appendices.tex
\section{Deep Dive into Theoretical Results}
\label{app:gen-theory}
\subsection{Intuition behind the axioms}
\label{app:axiom-theory}
The core idea is to change the notion of "a k-range task needs k-ranged information to be solved" to "a k-range task needs k-ranged information to be approximated". This is fundamental for designing long-range benchmarks as GNNs are NOT solvers in the algorithmic sense, they are approximators in the statistical learning sense. For this reason we have chosen joint distributions to formalize the notion of long-range for GNNs, as it enables us to rule out the existence of unknown shorter-ranged correlations that may aid in approximating the task. This phenomenon enabled shorter-ranged models to be presented as long-ranged, as we show in Section \ref{sec:audit}.

One example of this phenomenon is eccentricity, which has been used in ECHO to present a long-ranged task. Even though to solve the task you need to find the furthest node for each node, when restricting the problem to certain dataset, there likely exist shorter-ranged correlations GNNs can exploit. For example: it is intuitive to think that nodes with higher degrees will tend to have lower eccentricities. Consequently, it is unclear what performance can be achieved with shorter-ranged models. The clearest specific example we show of this is in Figure 2, where models have ranges around 3 but eccentricity is on average between 20 and 25 \cite{miglior2025can}.

\subsection{On the theoretical results of the Axioms}
\label{app:prop-theory}
These results provide a solid theoretical link between the absence of each axiom with undesirable behaviour that may undermine past and future research. These are essential to provide solid grounds for preventing future research from falling into the same pitfalls that have been repeated in the past. These results are to be seen as cautionary tales that prevent future research from unknowingly designing biased test cases to validate their long-ranged GNN developments.

Proposition \ref{prop:axiom2} shows that removing axiom 2 opens the door to tasks that will have arbitrarily large targets to be predicted depending on the graph size. So, “long-range” tasks designed on larger graphs without verifying Axiom 2 might be benchmarking the capabilities of GNNs of adapting to larger normed predictions instead of their long-range capabilities.

Proposition \ref{prop:axiom3} proposition is a formalization using our framework of a documented occurrence, since \cite{bambergermeasuring, luo2025can} show range-$\approx 2$ models matching long-range architectures on LRGB. 

To highlight the importance of Proposition \ref{prop:axiom4} we must point out previous work contextualize it. In \cite{bambergermeasuring}, to claim that a GNN learned the true range of a task, they used the $k$-power task, which is analogous to the one described in Proposition \ref{prop:axiom4}. Our proposition highlights that, had that validation in \cite{bambergermeasuring} been done on large enough graphs, the opposite conclusion could’ve been reached. As one could unknowingly generate a problem that is essentially trivial, solve it with a GNN that predicts a constant, thus concluding that the range proposed is meaningless.

In current benchmarks \cite{dwivedi2022long, zhou2025glora, miglior2025can} there is no objective notion of what is “good” performance. You can only conclude that model A beats model B, but never how far either one is from optimal, or how much of the remaining error is simply irreducible at a given range. Theorem \ref{thm:mle} fixes this by giving, for each observation range $k$, a closed-form error $\varepsilon_{\mathrm{MLE}}(k)$ that no model with effective range $\le k$ can beat — computable before training, and independent of the graph itself. This enables absolute calibration. Instead of Figure \ref{fig:grip_cycles} saying "the GCN reached MSE 0.42," we can say it achieves the information-theoretic optimum for ranges up to 5 and extracts essentially nothing from ranges 6–9, even though messages do reach those nodes. This calibration directly enables researchers to attribute a stronger performance to better long-range modelling. If a model's error drops below $\varepsilon_{\mathrm{MLE}}(k^\star-1)$, that's provable evidence it used range-$k$ information. "Is this improvement actually long-range?" stops being a rhetorical question and becomes a decidable one. This is precisely the kind of certificate that was missing in LRGB , and its absence is what made the reassessments \cite{tonshoff2023did, luo2025can, bambergermeasuring} necessary in the first place.

\section{$f$-stability's Motivation}
\label{app:f-stable}

In the main paper we propose $f$-stability as a concept which can be used to construct provably long-range benchmarks that satisfy axioms 1 through 4. Theorem \ref{thm:trip} gives sufficiency, $f$-stability is enough to satisfy the axioms. Here we add the near-converse 
\begin{proposition}
\label{prop:converse}
Fix the design template where features are i.i.d.\ and $y_u = f(\{x_v : v \in \Nh_k(u)\})$. Then $\Hent(Y_u \mid G_u^{k'}) = \Hent(Y_u)$ for $k' < k$, and Axiom~4 holds across all graph families if and only if $\Hent\big(f(\{X_1,\dots,X_n\})\big)$ is constant in $n$. If this entropy depends on $n$, Axiom~4 fails on any family with non-constant boundary sizes.
\end{proposition}
\begin{proof}
Let $n = |\Nh_k(u)|$. Nodes within distance $k' < k$ of $u$ are disjoint from $\Nh_k(u)$, so by the i.i.d.\ assumption $Y_u$ is independent of $G_u^{k'}$, and hence $\Hent(Y_u \mid G_u^{k'}) = \Hent(Y_u)$. Since the features in $\Nh_k(u)$ are $n$ i.i.d.\ samples of $X$, we have $\Hent(Y_u) = \Hent\big(f(\{X_1,\dots,X_n\})\big)$. Axiom~4 therefore requires this quantity to be the same for all boundary sizes $n$ occurring across graphs. If it is constant in $n$, Axiom~4 holds. If it depends on $n$, then any family containing two boundary sizes with different entropy violates Axiom~4.
\end{proof}

So within this template, constancy of the aggregate's law in $n$ is forced by Axiom 4, and $f$-stability is the natural (and, up to entropy-equality, necessary) way to obtain it.

\section{Proofs and Lemmas}
\label{app:proofs}
This appendix contains all the proofs of the theoretical results. To do so, we first prove a series of properties that will be used throughout the proofs.
\subsection{Auxiliary results}

\begin{proposition}
    Given two independent Normal random variables $X_1 \sim N(\mu_1,\sigma_1^2)$ and $X_2 \sim N(\mu_2,\sigma_2^2)$. Then $\alpha X_1+\beta X_2 \sim N(\alpha\mu_1 + \beta \mu_2, \alpha^2\sigma_1^2 + \beta^2\sigma_2^2)$
\end{proposition}
\textbf{Proof.}
Recall that the MGF of a normally distributed random variable $X \sim N(\mu, \sigma^2)$ is:
    \[ M_X(t) = \exp\left(\mu t + \frac{1}{2}\sigma^2 t^2\right) \]
Let $Y = \alpha X_1 + \beta X_2$. We want to find the MGF of $Y$, which is $M_Y(t) = E[e^{tY}]$. 
Substituting our expression for $Y$:
    \[ M_Y(t) = E\left[e^{t(\alpha X_1 + \beta X_2)}\right] = E\left[e^{(\alpha t)X_1} e^{(\beta t)X_2}\right] \]

Because $X_1$ and $X_2$ are independent, the expectation of their product separates into the product of their expectations:
    \[ M_Y(t) = E\left[e^{(\alpha t)X_1}\right] E\left[e^{(\beta t)X_2}\right] = M_{X_1}(\alpha t) M_{X_2}(\beta t) \]

Now, substitute the specific normal MGF formulas for $X_1$ and $X_2$:
    \begin{align*}
        M_Y(t) &= \exp\left(\mu_1 (\alpha t) + \frac{1}{2}\sigma_1^2 (\alpha t)^2\right) \exp\left(\mu_2 (\beta t) + \frac{1}{2}\sigma_2^2 (\beta t)^2\right) \\
        &= \exp\left( (\alpha\mu_1)t + \frac{1}{2}(\alpha^2\sigma_1^2)t^2 + (\beta\mu_2)t + \frac{1}{2}(\beta^2\sigma_2^2)t^2 \right) \\
        &= \exp\left( (\alpha\mu_1 + \beta\mu_2)t + \frac{1}{2}(\alpha^2\sigma_1^2 + \beta^2\sigma_2^2)t^2 \right)
    \end{align*}

We recognize this final expression as the exact MGF of a normal distribution with mean $\mu^* = \alpha\mu_1 + \beta\mu_2$ and variance $(\sigma^*)^2 = \alpha^2\sigma_1^2 + \beta^2\sigma_2^2$.

Since MGFs are unique, it follows exactly that:
    \[ \alpha X_1 + \beta X_2 \sim N(\alpha\mu_1 + \beta \mu_2, \alpha^2\sigma_1^2 + \beta^2\sigma_2^2) \]
\begin{proposition}
    Given $n$ i.i.d Gumbel random variables $X_i \sim \text{Gumbel}(\mu, \beta)$ then $\max_{i=1,...,n}\{X_i\} \sim \text{Gumbel}(\mu+\beta \ln(n),\beta)$.
    \label{prop:gumbel1}
\end{proposition}
\textbf{Proof.}
Let $X_1, X_2, \dots, X_n$ be i.i.d. Gumbel random variables with the cumulative distribution function (CDF):
$$F(x; \mu, \beta) = \exp\left( -e^{-(x-\mu)/\beta} \right)$$

Let $Y = \max\{X_1, \dots, X_n\}$. Since the variables are independent and identically distributed, the CDF of the maximum is:
$$F_Y(y) = P(\max\{X_1, \dots, X_n\} \leq y) = \prod_{i=1}^n P(X_i \leq y) = [F(y; \mu, \beta)]^n$$

Substituting the Gumbel CDF:
$$F_Y(y) = \left[ \exp\left( -e^{-(y-\mu)/\beta} \right) \right]^n$$
$$F_Y(y) = \exp\left( -n \cdot e^{-(y-\mu)/\beta} \right)$$

We can rewrite $n$ as $e^{\ln(n)}$:
$$F_Y(y) = \exp\left( -e^{\ln(n)} \cdot e^{-(y-\mu)/\beta} \right)$$
$$F_Y(y) = \exp\left( -e^{-(y - \mu - \beta \ln(n))/\beta} \right)$$

Comparing this to the standard form $F(y; \mu', \beta)$, we identify the new location parameter:
$$\mu' = \mu + \beta \ln(n)$$

Thus, $Y \sim \text{Gumbel}(\mu + \beta \ln(n), \beta)$. $\qed$

\begin{proposition}
    Given $n$ i.i.d Cauchy random variables $X \sim \text{Cauchy}(x_0, \gamma)$, then $\frac{X_1+ \cdot \cdot \cdot X_n}{n} \sim \text{Cauchy}(x_0, \gamma)$.
\end{proposition}
\textbf{Proof.}
Let $X_1, X_2, \dots, X_n$ be i.i.d. random variables with $X_i \sim \text{Cauchy}(x_0, \gamma)$. 
The characteristic function of a Cauchy distribution is given by:
$$\varphi_X(t) = E[e^{itX}] = e^{itx_0 - \gamma|t|}$$

Let $\bar{X} = \frac{1}{n} \sum_{j=1}^{n} X_j$. We find the characteristic function of $\bar{X}$:
\begin{align*}
\varphi_{\bar{X}}(t) &= E\left[e^{it\left(\frac{1}{n}\sum X_j\right)}\right] \\
&= E\left[\prod_{j=1}^{n} e^{i(t/n)X_j}\right]
\end{align*}

Since the $X_j$ are independent, the expectation of the product is the product of the expectations:
\begin{align*}
\varphi_{\bar{X}}(t) &= \prod_{j=1}^{n} E\left[e^{i(t/n)X_j}\right] \\
&= \prod_{j=1}^{n} \varphi_{X_j}\left(\frac{t}{n}\right)
\end{align*}

Substituting the Cauchy characteristic function for each $X_j$:
\begin{align*}
\varphi_{\bar{X}}(t) &= \prod_{j=1}^{n} e^{i(t/n)x_0 - \gamma|t/n|} \\
&= \left(e^{i(t/n)x_0 - \gamma|t/n|}\right)^n \\
&= e^{n \left(i\frac{t}{n}x_0 - \gamma\frac{|t|}{n}\right)} \\
&= e^{itx_0 - \gamma|t|}
\end{align*}

The resulting characteristic function $\varphi_{\bar{X}}(t) = e^{itx_0 - \gamma|t|}$ is exactly the characteristic function of a $\text{Cauchy}(x_0, \gamma)$ distribution. By the uniqueness of characteristic functions, we conclude:
$$\frac{X_1 + \dots + X_n}{n} \sim \text{Cauchy}(x_0, \gamma)$$ 
$\qed$

Another resource that will be used throughout is the following:
\begin{proposition}
    Given an $n,k \in \N$, there exists $G_{n,k}$ where $\forall u \in G_{n,k}$, $|\{v \in G_{n,k} | d_{G_{n,k}}(u,v) = k\}| = n$.
    \label{prop:graphs}
\end{proposition}
\textbf{Proof. } Given $n,k\in \N$, we set $V = \{1,...,2n(k+1)\}$, and $E = \{(u,v) | u - v \equiv 1 \text{ mod } 2(k+1)\}$, in Figure \ref{fig:graph_donut}, we can observe how such a graph may look like. By definition, it is immediate that $|\{v \in G_{n,k} | d_{G_{n,k}}(u,v) = k\}| = n$. $\qed$
\newline
\begin{figure}
    \centering
    \includegraphics[width=0.5\linewidth]{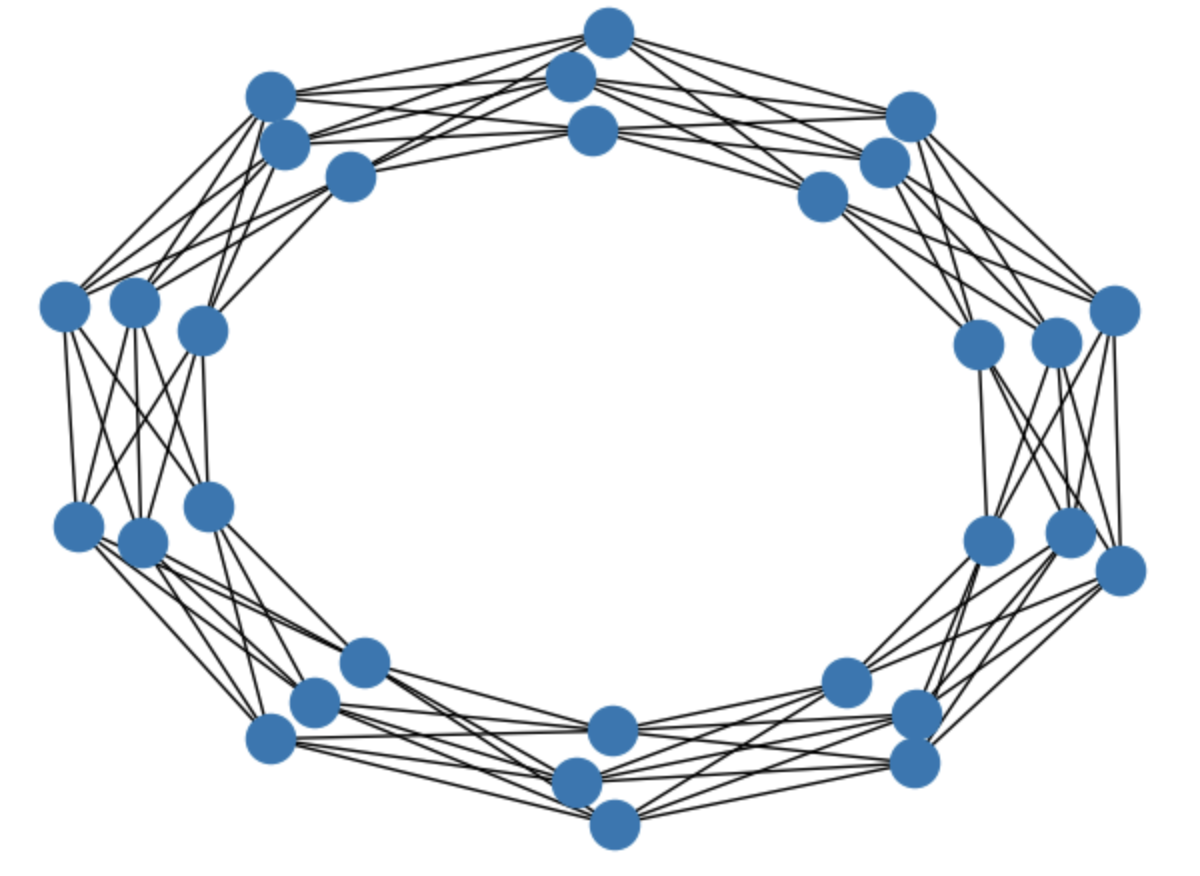}
    \caption{Example of the graph used in the Proposition \ref{prop:graphs} for $n = 3, k=4$}
    \label{fig:graph_donut}
\end{figure}
\subsection{Necessity of the axioms}

\renewcommand{\thesection}{3}
\setcounter{theorem}{1}
\begin{proposition}[Tightness is necessary]
Fix $k > 1$ and let $X \sim \mathcal{N}(0,1)$. Let $\{G_n\}$ be a sequence of graphs in which every node $u$ satisfies $|\Nh_k(u)| \ge n$, and define $Y_u = \max_{v \in \Nh_k(u)} X_v$. The construction satisfies Axioms~1, 3, and 4, but the family $\{Y_u\}$ is not tight: for every $M \in \R$, $\Prob(Y_u > M) \to 1$ as $n \to \infty$.
\end{proposition}
\renewcommand{\thesection}{\Alph{section}}
\textbf{Proof.} Axiom 1 is satisfied by definition. Axiom 3 is also satisfied as the predictions only rely on the samples at distance exactly $k$ and are independent from the rest. By Proposition \ref{prop:gumbel1} (applied to the Gumbel limit of the standard normal maximum), $Y_u$ is stochastically dominated below by a variable whose location parameter grows without bound as $n \to \infty$. This verifies Axiom~4 as all entropies will be equal since the only change in distribution is a shift in the expected value. But $\lim_{n \rightarrow \infty} P(Y_u > M) = 1$ because $X$ being unbounded above implies $P(X > M) > 0$ for all $M$, so the maximum of $n$ i.i.d.\ copies exceeds $M$ with probability $1 - (1 - P(X>M))^n \to 1$. $\qed$

In the case that a bounded distribution had been chosen, Axiom~4 cannot be satisfied due to the following proposition.

\begin{proposition}
    \label{prop:max-finite}
    Given $X$ a random variable with CDF $F$ such that its support is contained in $[a,b]$ for some $a,b \in \R$, its second order moment exists. Let's consider the random variable $Y_n = \max\{X_1,...,X_n\}$ where $X_i$ are i.i.d copies of $X$. Then for all $n_1$ there exists $n_2$ such that $\Hent(Y_{n_1}) \neq \Hent(Y_{n_2})$
\end{proposition}
\textbf{Proof.} To prove this we will use the following known bound for entropy: given any continuous random variable $Y$ with fixed variance $\text{var}(Y) = \sigma^2$, then $\Hent(Y) \leq \frac{1}{2}\log(2\pi e\sigma^2)$. With this in mind, if we now prove that $\lim_{n \rightarrow \infty} \text{var}(M_n) = 0$ we will get our desired result, because given any $Y_{n_1}$ we can find $n_2$ such that $\Hent(Y_{n_2}) < \log(2\pi e \sigma_{n_2}^2) < \Hent(Y_{n_1})$. So let's prove that $Y_n$ converges in probability to the constant $b$, as convergence in probability implies convergence of all moments, so if $Y_n$ converges to a constant, its variance will converge to zero. Indeed,
\begin{equation}
    \begin{split}
        \lim_{n \rightarrow \infty} P(|Y_n - b| > \varepsilon) & = \lim_{n \rightarrow \infty} P(b - Y_n > \varepsilon) = \lim_{n \rightarrow \infty} P(Y_n < b-\varepsilon) \\
         & = \lim_{n \rightarrow \infty} F(b-\varepsilon)^n = 0
    \end{split}
\end{equation}
Where $F$ is the cumulative density function that satisfies $F(x) < 1$ for all $x < b$, otherwise the upper bound would be lesser. This concludes our proof as $M_n$ converges in probability to a constant, which implies that the variance converges to 0, which allows us to always obtain a variable with less entropy. $\qed$

Essentially, using the example introduced in Proposition \ref{prop:graphs}, given an arbitrary graph $G$ and node $u \in G$ we can find another graph with nodes where the entropy of the predictions will be radically different. Thus arbitrarily biasing the benchmark.

\renewcommand{\thesection}{3}
\setcounter{theorem}{2}
\begin{proposition}[Strict $k$-Range is necessary]
Suppose a task satisfies Axioms~1, 2, 4 with target range $k$, but there exists $i < k$ with $\Hent(Y_u \mid G_u^i) < \Hent(Y_u \mid G_u^{i-1})$. Let $f^\star_i$ be the optimal range-$i$ predictor under any continuous error measure $\mathrm{err}$. Considering a finite set of known range-$k$ functions, and $f$ the one where the error measure is minimum, whenever $\mathrm{err}(f) > \mathrm{err}(f^\star_i)$, an improvement attained at range $i$ from $f^\star_{i-1}$ to $f^\star_i$ is reportable as an improvement at range $k$.
\end{proposition}
\renewcommand{\thesection}{\Alph{section}}
\textbf{Proof. } The proof is immediate, as $\Hent(Y_u \mid G_u^{i}) < \Hent(Y_u \mid G_u^{i-1})$ implies that a model with range $i-1$ cannot achieve $0$ error as there is information beyond that range. Let's denote $\varepsilon_{i-1}$ the infimum of the error achievable by a model that only accounts for information on $i-1$ hops, and $\varepsilon_{i}$ the infimum of a model accounting for information on $i$ hops. Then $\varepsilon_i < \varepsilon_{i-1}$. Consequently, if the task is claimed to be of range $k$, then a model of range $i$ can improve on a model of range $i-1$ while still getting reported as an improvement on a range-$k$ task. $\qed$

\renewcommand{\thesection}{3}
\setcounter{theorem}{3}
\begin{proposition}[Topology-Invariance is necessary]
Fix $k > 1$. Sample $x_u \sim X$ i.i.d.\ from any distribution with finite variance, and let $y_u$ be the average of features at distance exactly $k$ from $u$. The construction satisfies Axioms~1, 2, 3, but for every $\varepsilon$ there exist graphs $G$ on which $\Hent(Y_u) < \varepsilon$.
\end{proposition}
\renewcommand{\thesection}{\Alph{section}}
\textbf{Proof. } So $y_u$ is defined as an average of i.i.d random variables with finite mean and finite variance. As a direct consequence of the Central Limit Theorem, $y_u$ will converge to the constant random variable with value $\mu$ as the size of the average grows, so its variance converges to $0$. Using again the upper bound of the entropy used in the proof of \ref{prop:max-finite}, we can deduce the entropy will tend to $-\infty$. As such, for any $\varepsilon$ we can find a graph with sufficiently large sets of nodes at distance exactly $k$ such that $\Hent(Y_u) < \varepsilon$. $\qed$

\subsection{TRIP and GRIP}
\renewcommand{\thesection}{5}
\setcounter{theorem}{1}
\begin{proposition}[Closure properties]
Let $X$ be $f$-stable and $\tilde X$ be $\tilde f$-stable.
\begin{enumerate}
    \item For any measurable $g \colon \R \to \R$, $X$ is $(g \circ f)$-stable.
    \item The pair $(X, \tilde X)$ is $F$-stable with $F(\{X_i\}, \{\tilde X_j\}) = (f(\{X_i\}), \tilde f(\{\tilde X_j\}))$.
\end{enumerate}
\end{proposition}
\renewcommand{\thesection}{\Alph{section}}
\textbf{Proof.} For $X$ to be $g \circ f$-stable, we need to verify that $(g \circ f) (\{X_1,...,X_n\}) \overset{d}{=} Z$ for some fixed random variable $Z$. Indeed, $(g \circ f) (\{X_1,...,X_n\}) = g(Y)$ for some $Y$ because $X$ is $f$-stable. Then, as $g$ is measurable, $g(Y)$ is a random variable $Z$. The second point is immediate by definition. $\qed$

\renewcommand{\thesection}{5}
\setcounter{theorem}{3}
\begin{theorem}[\TRIP{} satisfies all four axioms]
For any graph $G$ and any $f$-stable pair $(X, f)$, the \TRIP{} task at range $k$ satisfies Axioms~1--4. In particular, the marginal $Y_u$ and the entropy $\Hent(Y_u)$ are graph-independent constants determined only by $(X, f)$.
\end{theorem}
\renewcommand{\thesection}{\Alph{section}}
\textbf{Proof.} Axiom~1 is immediate from the definition of \TRIP{}. By $f$-stability, $Y_u \overset{d}{=} Y$ for all $G$ and $u$, where $Y = f(\{X_i\}_{i=1}^n)$ is the same random variable for every $n \ge 1$; this yields Axiom~2 (the singleton family $\{\mathcal{L}(Y)\}$ is automatically tight) and Axiom~4 (the conditional entropy below the target range only sees i.i.d.\ features whose distribution does not depend on $(G,u)$). For Axiom~3: $\Nh_k(u)$ is disjoint from $G_u^{k-1}$ and the $X_v$ are i.i.d.\ across nodes, so $Y_u$ is independent of every $G_u^{k'}$ with $k' < k$, giving $\mathrm{I}(Y_u; G_u^{k'}) = 0$. Conversely, $Y_u$ is a measurable function of features in $\Nh_k(u) \subseteq G_u^k$, so $\mathrm{I}(Y_u; G_u^k) > 0$ whenever $f$ is non-constant. The marginal-constant claim follows from $\mathcal{L}(Y_u) = \mathcal{L}(Y)$. $\qed$
\renewcommand{\thesection}{5}
\setcounter{theorem}{5}
\begin{theorem}[Per-range MLE oracle]
Fix a \GRIP{} task with parameters $\{(X_i, f_i, K_i)\}_i$ and read-out $h$, and let $k^\star \ge 0$ be a target observation range. Let $\hat X_i$ be the MLE realisation of $X_i$. The Maximum-Likelihood predictor of $Y_u$ given $G_u^{k^\star}$ is $h(\tilde y_u)$, where for each $(i,k)$,
\[
\tilde y_u^{(i,k)} =
\begin{cases}
f_i\big(\{(x_v)_i : v \in \Nh_k(u)\}\big) & \text{if } k \le k^\star,\\
f_i\big(\{\hat X_i\}\big) & \text{if } k > k^\star.
\end{cases}
\]
The expected error of this predictor is a graph-independent function of $(\{(X_i, f_i, K_i)\}_i, h, k^\star)$ alone, and provides a lower bound for any model with effective range $\le k^\star$.
\end{theorem}
\renewcommand{\thesection}{\Alph{section}}
\textbf{Proof.} Since for the ranges conditioned on $G_u^{k^\star}$ the Maximum Likelihood estimator is the actual known values. The values at ranges greater than $k^\star$ are independent of the rest of the values; as such we obtain their MLE estimator through a direct application of the invariance principle for maximum likelihood estimators described in Pages 319--323 of \cite{casella2024statistical}, as we are just composing a function. The same invariance principle is applied to go from $\tilde{y}_u$ to $h(\tilde{y}_u)$. Graph-independence of the expected error follows from Axioms~2 and 4, which guarantee that $Y_u$ and all conditional distributions below the target range are graph-independent. $\qed$

\section{Long-rangedness and homophily}
\label{app:homo}

Throughout this work we assume i.i.d sampled features. Thus enforcing low feature homophily. In this section we look into the interplay between feature homophily and our notion of long-rangedness through a simple example.

Let's consider the following graph:

$$node_1 --- (node_1, node_4, node_5, ..., node_n) --- node_{n+1}$$

Where $node_1$ and $node_{n+1}$ are connected to the nodes $2,...,n$.

Let's consider a generic range-2 interaction, that is, each node is a function of nodes at distance 2. And we name $X_i$ the features attached to node i. Also, let’s assume that for each node the correlation, i.e covariance between the two variables divided by product of variances, with another node is greater the closer it is in the graph. That is, nodes at distance 1 have greater correlation than those at distance 2 and so on. Since non-zero correlation implies non-zero mutual information, we will now work in terms of mutual information. Now we split this into two cases according to whether the following chain-rule inequality is strict or not, $I(X_1; X_2,...,X_n) \geq I(X_1;X_2,...,X_{n-1})$. If $I(X_1; X_2,...,X_n) > I(X_1;X_2,...,X_n-1)$, then the more nodes we add, the n-1 nodes in the middle would contain more and more information on nodes 1 and n+1. This implies that how well $node_1$ could infer $node_n+1$’s features with range-1 information would depend on graph size, thus introducing structural bias. If $I(X_1; X_2,...,X_n) = I(X_1;X_2,...,X_{n-1})$, it implies that all $X_2,...,X_{n-1}$ contain the same information on $X_1$. This would introduce a structural bias, as node’s 1 and $n+1$ would only need to extract information from exactly ONE node, thus ignoring the actual structure of the graph. This very simple example showcases how adding correlations to the setting can cause: structural biases and the range of the interactions to decrease. 

This example suggests that homophily may be intrinsically related to short-range structure, as correlation (homophily) in the data may collapse long-ranged information on smaller effective ranges. Nevertheless, the example is illustrative rather than a general proof, and we believe there remains a not-yet-fully-understood interplay between homophily and range that merits further investigation.

\section{Data Generation}
Generation is deterministic given (topology collection, $k$, $(X,f)$, seed). \href{https://www.kaggle.com/datasets/ferranh/trip-grip-sampled-data}{https://www.kaggle.com/datasets/ferranh/trip-grip-sampled-data} contains all the data used for sections 6.1 and 6.2 and a README.md with instructions on where each file should be placed in order to run the experiments.
\section{Inter-graph comparisons ablations}
\label{app:graph-comparison}
To compare the results across graphs there is one thing that must be taken into account. In the $(X,f)$ pairs such that $X$ is $f$-stable, $f$ contains terms which are just the features of the other nodes, and also the size $n$ of the set of nodes at distance exactly $k$. In graphs where $n$ is constant, such as regular graphs, the TRIP problem is easier to fit by a GNN, as the $n$ term does not change across different nodes. To make the comparison between two graphs fair it is necessary to compute the distribution of these values and compute its entropy. To ensure that the amount of bits necessary to capture them is the same between the graphs that we want to compare. In the following tables we report the entropies of all the $n$ values of the graphs we have compared in the paper. That is: the graphs generated using the expected degree algorithm with their SDRF processed version, as well as the charge and synthetic topologies of ECHO \cite{miglior2025can}. We observe that in the ECHO datasets the entropy difference is, at most 0.7 bits. In the SDRF case there are some larger differences in the cases of 6 and 7, which were appropriately discussed in the main manuscript.

\begin{table}[!ht]
    \centering
    \begin{tabular}{ccccccccc}
    \hline
         & ~ & ~ & ~ & k & ~ & ~ & ~ \\ \hline
        ~ & 1 & 2 & 3 & 4 & 5 & 6 & 7 & 8 \\ \hline
        synth & 2.04 & 3.02 & 3.44 & 3.68 & 3.84 & 3.95 & 4.02 & 4.06 \\ 
        charge & 1.54 & 2.52 & 2.87 & 3.17 & 3.3 & 3.33 & 3.33 & 3.27 \\ \hline 
    \end{tabular}
    \caption{This table shows the entropy of the distribution across all nodes of the ECHO synth and charge topologies of the amount of nodes at distance $k$. }
\end{table}

\begin{table}[!ht]
    \centering
    \begin{tabular}{cccccccc}
    \hline
        ~ & ~ & ~ & k & ~ & ~ & ~ \\ \hline
        ~ & 1 & 2 & 3 & 4 & 5 & 6 & 7 \\ \hline
        base & 3.83 & 6.68 & 8.46 & 7.56 & 8.29 & 5.46 & 1.78 \\ 
        SDRF & 4.39 & 7.43 & 8.29 & 8.24 & 7.27 & 2.09 & 0.04 \\ \hline
    \end{tabular}
    \caption{This table shows the entropy of the distribution across all nodes of the randomly generated graphs and their SDRF processed versions of the amount of nodes at distance $k$.}
\end{table}

\section{Limitations and Future Work}
\label{app:limitations}
To provide the theoretical guarantees described in this work it is necessary to use synthetic data. For this reason, future work directions mostly involve closing the gap between this synthetic framework and real-world settings. For example, further work on the $f$-stable notion of distributions may reveal ways to build kernels to simulate real-world data. Another potential line of future research is to relax Axiom~1 to make $g$ a non-deterministic function to study the effect of noise in long-range interactions modeling.

As mentioned in Appendix \ref{app:graph-comparison}, it is necessary to measure the entropy of certain distributions per graph in order to compare how hard it is to model long-range interactions under those settings. We believe that the mentioned entropy is a good proxy for the complexity of real-world graphs, as the regular topologies will be easier to model than those with varied connectivity patterns. This is in line with the work presented in \cite{bechler2024graph}.

\section{Hyperparameter configurations}
\label{app:hyper-param}
\subsection{Synthetic configurations}
For the TRIP experiments on cycles we used hidden dimension of 16, a learning rate of 0.001 and 2000 epochs. For the GRIP experiments on cycles we used the same learning rate and number of epochs but increased the hidden dimension to 64. Running all the TRIP and GRIP experiments shown in \Cref{sec:tripgrip} had a runtime of roughly 10 hours on an Apple M5 chip with 16GB of RAM.
\subsection{Curvature analysis}
To generate the graphs an expected degree algorithm was used. In the following manner:

\begin{algorithm}[H]
\caption{Graph Generation via Expected Degree Model}
\KwIn{Graph size $n$, degree support $\mathcal{S} = \{d_{\min}, \dots, d_{\max}\}$, target mean $\mu$, std dev $\sigma$}
\KwOut{Connected graph $G$}

\BlankLine
\tcp{Step 1: Fit Boltzmann distribution over degree support}
Solve for center $c$ such that $\mathbb{E}_{p_c}[d] = \mu$, where
$p_c(d) \propto \exp\!\left(-\tfrac{1}{2}\left(\tfrac{d - c}{\sigma}\right)^2\right)$\;

\BlankLine
\tcp{Step 2: Sample a degree sequence}
Draw $d_1, d_2, \dots, d_m \overset{\text{i.i.d.}}{\sim} p_c$ over $\mathcal{S}$,
where $m \geq n$ accounts for nodes that will be lost when extracting the largest connected component\;

\BlankLine
\tcp{Step 3: Sample a graph from the degree sequence}
Construct graph $G'$ using the expected degree model: connect each pair of nodes $(i, j)$ independently with probability
$p_{ij} = \frac{d_i \, d_j}{\sum_k d_k}$\;

\BlankLine
\tcp{Step 4: Extract the largest connected component}
Identify all connected components $\mathcal{C}_1, \mathcal{C}_2, \dots$ of $G'$\;
Set $G \leftarrow G'\!\left[\,\arg\max_{l} |\mathcal{C}_l|\,\right]$ \tcp*{subgraph induced by largest component}

\BlankLine
\Return{$G$}\;
\end{algorithm}

For the curvature experiments we have used a learning rate of 0.001, hidden dimension of 32 and 2000 epochs to train a standard GCN. The SDRF algorithm was run with $1000$ loops, $0.5$ removal bound and $1.0$ as tau. All of these experiments were run on an Apple M5 chip. The runtimes of the various parts were the following: applying SDRF had a runtime of approximately 2 hours and training the models had a runtime of roughly 8 hours.
\subsection{ECHO}
To make the comparison with the real ECHO benchmark fair, we have performed analogous hyperparameter sweeps at each range for each model. All the search intervals and settings are the same as those of \cite{miglior2025can} except for the number of layers, hidden dimension, learning rate and weight decay. Runtimes per epoch reported were on a single NVIDIA A100-SXM-80GB GPU. In the tables below we report the sweeps as well as the final configurations. To be mindful of computing resources we stopped running sweeps on a model once its performance was equal to that of the model predicting always $0$. Because the analysis done in \ref{app:graph-comparison} allows us to ascertain that the problems on the longer ranges will be as, if not more, difficult. Consequently, if a model fails to learn any significant signal at range $k$, it will also fail at ranges higher than $k$. The run times per epoch are reported through
\begin{table}[h]
\centering
\small
\caption{Hyperparameters and their best values across tasks PH-DGN.}\vspace{1mm}
\label{tab:integrated_hyperparams_phdgn}
\resizebox{\textwidth}{!}{%
\begin{tabular}{llllllllll}
\toprule
\textbf{Hyperparameter} & \textbf{Search interval} & synth-2 & synth-4 & synth-6 & synth-8 & charge-2 & charge-4 & charge-6 & charge-8  \\
\midrule
Number of layers & $[k-k+2]$ & 4 & 5 & 8 & 10 & 2 & 6 & 7 & 10 \\
Hidden dimension & $[16 - 128]$ & 123 & 65 & 90 & 108 & 37 & 128 & 109 & 101 \\
Learning rate & $[10^{-3} - 10^{-2}]$ & 0.0019829 & 0.0036591 & 0.0035931 & 0.0032598 & 0.0025752 & 0.00046393 & 0.001766 & 0.0016224 \\
Weight decay & $[10^{-8} - 10^{-13}]$ & 0.00017623 & 0.000649 & 0.00043574 & 0.00017166 & 0.000012821 & 0.00002681 & 0.000032552 & 0.000010521 \\
\midrule
$\epsilon$ & $[0.001 - 1.0]$ & 0.6 & 0.62677 & 0.63991 & 0.50392 & 0.96703 & 0.18725 & 0.48771 & 0.62536 \\
$\alpha$ & $[0.01 - 1.0]$ & 0.62968 & 0.79134 & 0.96059 & 0.20735 & 0.7506 & 0.019947 & 0.24087 & 0.34108\\
$\beta$ & $[0.01 - 1.0]$ & 0.62319 & 0.89197 & 0.012745 & 0.16402 & 0.85072 & 0.59061 & 0.79638 & 0.72338 \\
$p$ conv mode & $[\text{NaiveAggr}, \text{GCN}]$ & NaiveAggr & GCN & GCN & GCN & NaiveAggr & NaiveAggr & NaiveAggr & NaiveAggr \\
$q$ conv mode & $[\text{NaiveAggr}, \text{GCN}]$ & NaiveAggr & NaiveAggr & NaiveAggr & NaiveAggr & GCN & NaiveAggr & NaiveAggr & NaiveAggr \\
Doubled dimension & $[\text{True}, \text{False}]$ & True & False & False & False & True & True & True & False \\
Final state & $[p, q, pq]$ & \(q\) & \(p\) & \(pq\) & \(pq\) & \(pq\) & \(pq\) & \(p\) & \(q\) \\
Dampening mode & \makecell[l]{$[$param,\\param+,\\MLP4ReLU,\\DGNReLU,$]$} & param & param & param & MLP4ReLU & param & param & DGNReLU & MLP4ReLU
 \\
External mode & $[\text{MLP4Sin}, \text{DGNtanh}]$ & MLP4Sin & MLP4Sin & MLP4Sin & MLP4Sin & MLP4Sin & MLP4Sin & MLP4Sin & MLP4Sin \\
\midrule
Runtime/Epoch (s) & - & 1.97 & 1.88 & 1.99 & 2.20 & 6.73 & 11.42 & 7.30 & 10.03 \\
\bottomrule
\end{tabular}
}
\end{table}

\begin{table}[h]
\centering
\small
\caption{Hyperparameters and their best values across tasks for SWAN.}\vspace{1mm}
\vspace{1mm}
\label{tab:integrated_hyperparams_swan}
\resizebox{\textwidth}{!}{%
\begin{tabular}{llllllllll}
\toprule
\textbf{Hyperparameter} & \textbf{Search interval} & synth-2 & synth-4 & synth-6 & synth-8 & charge-2 & charge-4 & charge-6 & charge-8 \\
\midrule
Number of layers & \([k-k+2] \) & 4 & 4 & 8 & - & 2 & 6 & 8 & -\\
Hidden dimension & \([16 - 128] \) & 57 & 45 & 88 & - & 26 & 91 & 124 & - \\
Learning rate & $[10^{-3} - 10^{-2}]$ & 0.0050446 & 0.0058967 & 0.0028226 & - & 0.0023166 & 0.0011137 & 0.0061644 & - \\
Weight decay & $[10^{-13} - 10^{-3}]$ & 0.00013226 & 0.00034166 & 0.00099821 & - & 0.0000019995 & 0.00048812 & 0.00079115 & -\\
\midrule
$\epsilon$ & \([0.001 - 1.0]\) & 0.95111 & 0.94049 & 0.64641 & - & 0.75604 & 0.82622 & 0.60167 & -\\
$\gamma$ & \([0.001 - 0.5]\) & 0.27672 & 0.085026 & 0.1858 & - & 0.39086 & 0.38164 & 0.029949 & - \\
$\beta$ & \([-1.0 - 1.0]\) & -0.64706 & 0.81494 & -0.35339 & - & -0.018684 & -0.98335 & 0.87145 & - 
 \\[2mm]
Graph convolution & \makecell[l]{$[$AntiSymNaiveAggr (ASNA),\\BoundedGCNConv (BGC),\\BoundedNaiveAggr (BNA)$]$} & BNA & BGC & ASNA & - & ASNA & BGC & BNA & -\\\vspace{2mm}
Attention & $[$True, False$]$ & True & True & True & - & True & True & True & -\\
\midrule
Runtime/Epoch (s) & - & 1.95 & 1.86 & 0.89 & - & 6.74 & 3.78 & 1.64 & - \\
\bottomrule
\end{tabular}
}
\end{table}

\begin{table}[h]
\centering
\small
\caption{Hyperparameters and their best values across tasks for A-DGN.}
\vspace{1mm}
\label{tab:integrated_hyperparams_adgn}
\resizebox{\textwidth}{!}{%
\begin{tabular}{llllllllll}
\toprule
\textbf{Hyperparameter} & \textbf{Search interval} & synth-2 & synth-4 & synth-6 & synth-8 & charge-2 & charge-4 & charge-6 & charge-8 \\
\midrule
Number of layers & $[k - k+2]$ & 4 & 4 & - & - & 3 & 6 & - & - \\
Hidden dimension & $[16 - 128]$ & 111 & 96 & - & - & 103 & 32 & - & - \\
Learning rate & $[10^{-3} - 10^{-2}]$ & 0.007139 & 0.0052292 & - & - & 0.0041737 & 0.0098753 & - & - \\
Weight decay & $[10^{-13} - 10^{-3}]$ & 0.000010932 & 0.00063737 & - & - & 6.3409e-7 & 0.00082609 & - & - \\
 \\
\midrule
$\epsilon$ & $[0.001 - 0.5]$ & 0.47645 & 0.088879 & - & - & 0.49636 & 0.35951 & - & -
 \\
$\gamma$ & $[0.001 - 0.5]$ & 0.31699 & 0.27401 & - & - & 0.1865 & 0.29602 & - & -
 \\
Graph convolution & $[\text{NaiveAggr}, \text{GCN}]$ & GCN & GCN & - & - & GCN & GCN & - & -
 \\
\midrule
Runtime/Epoch (s) & - & 1.68 & 0.26 & - & - & 6.29 & 2.13 & - & - \\
\bottomrule
\end{tabular}
}
\end{table}

\begin{table}[h]
\centering
\small
\caption{Hyperparameters and their best values across tasks for GraphCON.}\vspace{1mm}
\label{tab:integrated_hyperparams_graphcon}
\resizebox{\textwidth}{!}{%
\begin{tabular}{llllllllll}
\toprule
\textbf{Hyperparameter} & \textbf{Search interval} & synth-2 & synth-4 & synth-6 & synth-8 & charge-2 & charge-4 & charge-6 & charge-8 \\
\midrule
Number of layers & $[k - k+2]$ & 4 & 5 & - & - & 4 & 5 & - & - \\
Hidden dimension & $[16 - 128]$ & 108 & 80 & - & - & 85 & 99 & - & - \\
Learning rate & $[10^{-3} - 10^{-2}]$ & 0.0056647 & 0.0071248 & - & - & 0.0068252 & 0.0086757 & - & - \\
Weight decay & $[10^{-13} - 10^{-3}]$ & 0.00047368 & 0.00018787 & - & - & 0.00058396 & 0.000039939 & - & - \\
\midrule
$\epsilon$ & $[0.001 - 1.0]$ & 0.84399 & 0.24835 & - & - & 0.86125 & 0.12498 & - & - \\
\midrule
Runtime/Epoch (s) & - & 1.80 & 0.10 & - & - & 6.25 & 2.98 & - & - \\
\bottomrule
\end{tabular}
}
\end{table}

\begin{table}[h]
\centering
\small
\caption{Hyperparameters and their best values across tasks for GCNII.\label{tab:integrated_hyperparams_gcn2}}
\vspace{1mm}
\resizebox{\textwidth}{!}{%
\begin{tabular}{llllllllll}
\toprule
\textbf{Hyperparameter} & \textbf{Search interval} & synth-2 & synth-4 & synth-6 & synth-8 & charge-2 & charge-4 & charge-6 & charge-8 \\
\midrule
Number of layers & $[k - k+2]$ & 4 & 6 & 7 & - & 4 & 5 & 8 & -
 \\
Hidden dimension & $[16 - 128]$ & 32 & 63 & 43 & - & 168 & 149 & 16 & - \\
Learning rate & $[10^{-3} - 10^{-2}]$ & 0.0049859 & 0.0021435 & 0.0082665 & - & 0.0018015 & 0.0012101 & 0.004821 & -\\
Weight decay & $[10^{-13} - 10^{-3}]$ & 0.00033804 & 0.00056113 & 0.00056882 & - & 0.00006297 & 0.000708 & 0.00082622 & - \\
\midrule
$\alpha$ & $[0.0 - 0.9]$ & 0.3413 & 0.1743 & 0.098132 & - & 0.56735 & 0.1486 & 0.8484 & -\\
\midrule
Runtime/Epoch (s) & - & 1.71 & 1.81 & 1.17 & - & 6.76 & 5.50 & 1.46 & - \\
\bottomrule
\end{tabular}
}
\end{table}

\begin{table}[h]
\centering
\small
\caption{Hyperparameters and their best values across tasks for GIN.}
\vspace{1mm}
\label{tab:integrated_hyperparams_gin}
\resizebox{\textwidth}{!}{%
\begin{tabular}{llllllllll}
\toprule
\textbf{Hyperparameter} & \textbf{Search interval} & synth-2 & synth-4 & synth-6 & synth-8 & charge-2 & charge-4 & charge-6 & charge-8 \\
\midrule
Number of layers & $[k - k+2]$ & 3 & 6 & 8 & - & 3 & 6 & 8 & -\\
Hidden dimension & $[16 - 128]$ & 54 & 22 & 29 & - & 80 & 17 & 26 & -\\
Learning rate & $[10^{-3} - 10^{-2}]$ & 0.0051168 & 0.0015522 & 0.0057223 & - & 0.0072599 & 0.0025787 & 0.0033032 & -\\
Weight decay & $[10^{-13} - 10^{-3}]$ & 0.00056302 & 0.00042337 & 0.00068561 & - & 0.00020653 & 0.00037468 & 0.0000010491 & -\\
\midrule
Runtime/Epoch (s) & - & 1.69 & 1.63 & 0.56 & 0.78 & 6.52 & 6.74 & 2.05 & - \\
\bottomrule
\end{tabular}
}
\end{table}

\begin{table}[h]
\centering
\small
\caption{Hyperparameters and their best values across tasks for GCN.}
\vspace{1mm}
\label{tab:integrated_hyperparams_gcn}
\resizebox{\textwidth}{!}{%
\begin{tabular}{llllllllll}
\toprule
\textbf{Hyperparameter} & \textbf{Search interval} & synth-2 & synth-4 & synth-6 & synth-8 & charge-2 & charge-4 & charge-6 & charge-8 \\
\midrule
Number of layers & $[k - k+2]$ & 4 & 6 & - & - & 4 & 6 & - & - \\
Hidden dimension & $[16 - 128]$ & 20 & 24 & - & - & 114 & 19 & - & - \\
Learning rate & $[10^{-3} - 10^{-2}]$ & 0.0092326 & 0.0013939 & - & - & 0.0027285 & 0.0011325 & - & - \\
Weight decay & $[10^{-13} - 10^{-3}]$ & 0.000018562 & 0.000089966 & - & - & 0.000025323 & 0.00038465 & - & - \\
\midrule
Runtime/Epoch (s) & - & 1.58 & 1.67 & - & - & 6.53 & 6.40 & - & - \\
\bottomrule
\end{tabular}
}
\end{table}
% \clearpage
\begin{table}[h]
\centering
\small
\caption{{Hyperparameters and their best values across tasks for GRIT.}}\vspace{1mm}
\label{tab:integrated_hyperparams_grit}
\resizebox{\textwidth}{!}{%
\begin{tabular}{llllllllll}
\toprule
\textbf{Hyperparameter} & \textbf{Search interval} & synth-2 & synth-4 & synth-6 & synth-8 & charge-2 & charge-4 & charge-6 & charge-8 \\
\midrule
Number of layers & $[k - k+2]$ & 4 & 5 & - & - &  3 & 5 & - & - \\
Hidden dimension & $[16, 32, 64, 96, 128]$ & 128 & 64 & - & - & 128 & 96 & - & - \\
Learning rate & $[10^{-4} - 10^{-2}]$ & 0.0009 & 0.0004285 & - & - & 0.00097063 & 0.00017332 & - & - \\
Weight decay & $[10^{-13} - 10^{-3}]$ & 0.000062203 & 0.00066628 & - & - & 0.00099632 & 0.00097794 & - & - \\
\midrule
Attention dropout & $[0 - 0.5]$ & 0.22791 & 0.12143 & - & - & 0.0022844 & 0.16741 & - & - \\
Number of attention heads &  $[2]$ & 2 & 2 & - & - & 2 & 2 & - & - \\
\midrule
Runtime/Epoch (s) & - & 1.2 & 212.0 & - & - & 10.33 & 871.00 & - & - \\
\bottomrule
\end{tabular}
}
\end{table}
\clearpage

\section{Code}
\label{app:code}
The code that contains the core functionalities involving TRIP and GRIP generation can be found in \href{https://github.com/ferranhernandezc/graph-grip}{https://github.com/ferranhernandezc/graph-grip}. Using this code the experiments on Sections \ref{sec:tripgrip} and \ref{sb-sct:curvature} can be reproduced. To reproduce the other experimental results it is necessary to install the mentioned code, and then use \href{https://github.com/ferranhernandezc/ECHO-grip}{https://github.com/ferranhernandezc/ECHO-grip}. The repositories contain further instructions.

\section{Disclosure of LLM usage}
\label{app:llm}
LLMs were used throughout the development of this project as a supporting tool for coding, which consisted of: reviewing human written code, writing of some unit tests and coding scripts to generate plots and run experiments.

%% file: checklist.tex
%%% END INSTRUCTIONS %%%

\begin{enumerate}

\item {\bf Claims}
    \item[] Question: Do the main claims made in the abstract and introduction accurately reflect the paper's contributions and scope?
    \item[] Answer: \answerYes{} % Replace by \answerYes{}, \answerNo{}, or \answerNA{}.
    \item[] Justification: The claims made in the abstract correspond to sections 3 through 6.
    \item[] Guidelines:
    \begin{itemize}
        \item The answer \answerNA{} means that the abstract and introduction do not include the claims made in the paper.
        \item The abstract and/or introduction should clearly state the claims made, including the contributions made in the paper and important assumptions and limitations. A \answerNo{} or \answerNA{} answer to this question will not be perceived well by the reviewers. 
        \item The claims made should match theoretical and experimental results, and reflect how much the results can be expected to generalize to other settings. 
        \item It is fine to include aspirational goals as motivation as long as it is clear that these goals are not attained by the paper. 
    \end{itemize}

\item {\bf Limitations}
    \item[] Question: Does the paper discuss the limitations of the work performed by the authors?
    \item[] Answer: \answerYes{} % Replace by \answerYes{}, \answerNo{}, or \answerNA{}.
    \item[] Justification: Appendix \ref{app:limitations} discusses the limitations of the proposed work.
    \item[] Guidelines:
    \begin{itemize}
        \item The answer \answerNA{} means that the paper has no limitation while the answer \answerNo{} means that the paper has limitations, but those are not discussed in the paper. 
        \item The authors are encouraged to create a separate ``Limitations'' section in their paper.
        \item The paper should point out any strong assumptions and how robust the results are to violations of these assumptions (e.g., independence assumptions, noiseless settings, model well-specification, asymptotic approximations only holding locally). The authors should reflect on how these assumptions might be violated in practice and what the implications would be.
        \item The authors should reflect on the scope of the claims made, e.g., if the approach was only tested on a few datasets or with a few runs. In general, empirical results often depend on implicit assumptions, which should be articulated.
        \item The authors should reflect on the factors that influence the performance of the approach. For example, a facial recognition algorithm may perform poorly when image resolution is low or images are taken in low lighting. Or a speech-to-text system might not be used reliably to provide closed captions for online lectures because it fails to handle technical jargon.
        \item The authors should discuss the computational efficiency of the proposed algorithms and how they scale with dataset size.
        \item If applicable, the authors should discuss possible limitations of their approach to address problems of privacy and fairness.
        \item While the authors might fear that complete honesty about limitations might be used by reviewers as grounds for rejection, a worse outcome might be that reviewers discover limitations that aren't acknowledged in the paper. The authors should use their best judgment and recognize that individual actions in favor of transparency play an important role in developing norms that preserve the integrity of the community. Reviewers will be specifically instructed to not penalize honesty concerning limitations.
    \end{itemize}

\item {\bf Theory assumptions and proofs}
    \item[] Question: For each theoretical result, does the paper provide the full set of assumptions and a complete (and correct) proof?
    \item[] Answer: \answerYes{} % Replace by \answerYes{}, \answerNo{}, or \answerNA{}.
    \item[] Justification: In Appendix \ref{app:proofs} full proofs and hypotheses are given.
    \item[] Guidelines:
    \begin{itemize}
        \item The answer \answerNA{} means that the paper does not include theoretical results. 
        \item All the theorems, formulas, and proofs in the paper should be numbered and cross-referenced.
        \item All assumptions should be clearly stated or referenced in the statement of any theorems.
        \item The proofs can either appear in the main paper or the supplemental material, but if they appear in the supplemental material, the authors are encouraged to provide a short proof sketch to provide intuition. 
        \item Inversely, any informal proof provided in the core of the paper should be complemented by formal proofs provided in appendix or supplemental material.
        \item Theorems and Lemmas that the proof relies upon should be properly referenced. 
    \end{itemize}

    \item {\bf Experimental result reproducibility}
    \item[] Question: Does the paper fully disclose all the information needed to reproduce the main experimental results of the paper to the extent that it affects the main claims and/or conclusions of the paper (regardless of whether the code and data are provided or not)?
    \item[] Answer: \answerYes{} % Replace by \answerYes{}, \answerNo{}, or \answerNA{}.
    \item[] Justification: All the code  is linked in Appendix \ref{app:code}
    \item[] Guidelines:
    \begin{itemize}
        \item The answer \answerNA{} means that the paper does not include experiments.
        \item If the paper includes experiments, a \answerNo{} answer to this question will not be perceived well by the reviewers: Making the paper reproducible is important, regardless of whether the code and data are provided or not.
        \item If the contribution is a dataset and\slash or model, the authors should describe the steps taken to make their results reproducible or verifiable. 
        \item Depending on the contribution, reproducibility can be accomplished in various ways. For example, if the contribution is a novel architecture, describing the architecture fully might suffice, or if the contribution is a specific model and empirical evaluation, it may be necessary to either make it possible for others to replicate the model with the same dataset, or provide access to the model. In general. releasing code and data is often one good way to accomplish this, but reproducibility can also be provided via detailed instructions for how to replicate the results, access to a hosted model (e.g., in the case of a large language model), releasing of a model checkpoint, or other means that are appropriate to the research performed.
        \item While NeurIPS does not require releasing code, the conference does require all submissions to provide some reasonable avenue for reproducibility, which may depend on the nature of the contribution. For example
        \begin{enumerate}
            \item If the contribution is primarily a new algorithm, the paper should make it clear how to reproduce that algorithm.
            \item If the contribution is primarily a new model architecture, the paper should describe the architecture clearly and fully.
            \item If the contribution is a new model (e.g., a large language model), then there should either be a way to access this model for reproducing the results or a way to reproduce the model (e.g., with an open-source dataset or instructions for how to construct the dataset).
            \item We recognize that reproducibility may be tricky in some cases, in which case authors are welcome to describe the particular way they provide for reproducibility. In the case of closed-source models, it may be that access to the model is limited in some way (e.g., to registered users), but it should be possible for other researchers to have some path to reproducing or verifying the results.
        \end{enumerate}
    \end{itemize}

\item {\bf Open access to data and code}
    \item[] Question: Does the paper provide open access to the data and code, with sufficient instructions to faithfully reproduce the main experimental results, as described in supplemental material?
    \item[] Answer: \answerYes{} % Replace by \answerYes{}, \answerNo{}, or \answerNA{}.
    \item[] Justification: Links to the anonymized code can be found in Appendix \ref{app:code}. In the repositories there are instructions on how to obtain the results described on the paper.
    \item[] Guidelines:
    \begin{itemize}
        \item The answer \answerNA{} means that paper does not include experiments requiring code.
        \item Please see the NeurIPS code and data submission guidelines (\url{https://neurips.cc/public/guides/CodeSubmissionPolicy}) for more details.
        \item While we encourage the release of code and data, we understand that this might not be possible, so \answerNo{} is an acceptable answer. Papers cannot be rejected simply for not including code, unless this is central to the contribution (e.g., for a new open-source benchmark).
        \item The instructions should contain the exact command and environment needed to run to reproduce the results. See the NeurIPS code and data submission guidelines (\url{https://neurips.cc/public/guides/CodeSubmissionPolicy}) for more details.
        \item The authors should provide instructions on data access and preparation, including how to access the raw data, preprocessed data, intermediate data, and generated data, etc.
        \item The authors should provide scripts to reproduce all experimental results for the new proposed method and baselines. If only a subset of experiments are reproducible, they should state which ones are omitted from the script and why.
        \item At submission time, to preserve anonymity, the authors should release anonymized versions (if applicable).
        \item Providing as much information as possible in supplemental material (appended to the paper) is recommended, but including URLs to data and code is permitted.
    \end{itemize}

\item {\bf Experimental setting/details}
    \item[] Question: Does the paper specify all the training and test details (e.g., data splits, hyperparameters, how they were chosen, type of optimizer) necessary to understand the results?
    \item[] Answer: \answerYes{} % Replace by \answerYes{}, \answerNo{}, or \answerNA{}.
    \item[] Justification: All the hyperparameter details and experimental setup is detailed either in the main paper or in Appendix \ref{app:hyper-param}.
    \item[] Guidelines:
    \begin{itemize}
        \item The answer \answerNA{} means that the paper does not include experiments.
        \item The experimental setting should be presented in the core of the paper to a level of detail that is necessary to appreciate the results and make sense of them.
        \item The full details can be provided either with the code, in appendix, or as supplemental material.
    \end{itemize}

\item {\bf Experiment statistical significance}
    \item[] Question: Does the paper report error bars suitably and correctly defined or other appropriate information about the statistical significance of the experiments?
    \item[] Answer: \answerYes{} % Replace by \answerYes{}, \answerNo{}, or \answerNA{}.
    \item[] Justification: All appropiate results are reported with standard deviations using cross-validation over various splits.
    \item[] Guidelines:
    \begin{itemize}
        \item The answer \answerNA{} means that the paper does not include experiments.
        \item The authors should answer \answerYes{} if the results are accompanied by error bars, confidence intervals, or statistical significance tests, at least for the experiments that support the main claims of the paper.
        \item The factors of variability that the error bars are capturing should be clearly stated (for example, train/test split, initialization, random drawing of some parameter, or overall run with given experimental conditions).
        \item The method for calculating the error bars should be explained (closed form formula, call to a library function, bootstrap, etc.)
        \item The assumptions made should be given (e.g., Normally distributed errors).
        \item It should be clear whether the error bar is the standard deviation or the standard error of the mean.
        \item It is OK to report 1-sigma error bars, but one should state it. The authors should preferably report a 2-sigma error bar than state that they have a 96\% CI, if the hypothesis of Normality of errors is not verified.
        \item For asymmetric distributions, the authors should be careful not to show in tables or figures symmetric error bars that would yield results that are out of range (e.g., negative error rates).
        \item If error bars are reported in tables or plots, the authors should explain in the text how they were calculated and reference the corresponding figures or tables in the text.
    \end{itemize}

\item {\bf Experiments compute resources}
    \item[] Question: For each experiment, does the paper provide sufficient information on the computer resources (type of compute workers, memory, time of execution) needed to reproduce the experiments?
    \item[] Answer: \answerYes{} % Replace by \answerYes{}, \answerNo{}, or \answerNA{}.
    \item[] Justification: All the details are given in Appendix \ref{app:hyper-param}.
    \item[] Guidelines:
    \begin{itemize}
        \item The answer \answerNA{} means that the paper does not include experiments.
        \item The paper should indicate the type of compute workers CPU or GPU, internal cluster, or cloud provider, including relevant memory and storage.
        \item The paper should provide the amount of compute required for each of the individual experimental runs as well as estimate the total compute. 
        \item The paper should disclose whether the full research project required more compute than the experiments reported in the paper (e.g., preliminary or failed experiments that didn't make it into the paper). 
    \end{itemize}
    
\item {\bf Code of ethics}
    \item[] Question: Does the research conducted in the paper conform, in every respect, with the NeurIPS Code of Ethics \url{https://neurips.cc/public/EthicsGuidelines}?
    \item[] Answer: \answerYes{} % Replace by \answerYes{}, \answerNo{}, or \answerNA{}.
    \item[] Justification: The authors have read and abided by the Ethics Guidelines.
    \item[] Guidelines:
    \begin{itemize}
        \item The answer \answerNA{} means that the authors have not reviewed the NeurIPS Code of Ethics.
        \item If the authors answer \answerNo, they should explain the special circumstances that require a deviation from the Code of Ethics.
        \item The authors should make sure to preserve anonymity (e.g., if there is a special consideration due to laws or regulations in their jurisdiction).
    \end{itemize}

\item {\bf Broader impacts}
    \item[] Question: Does the paper discuss both potential positive societal impacts and negative societal impacts of the work performed?
    \item[] Answer: \answerNA{} % Replace by \answerYes{}, \answerNo{}, or \answerNA{}.
    \item[] Justification: The current work focuses on general and theoretic properties of Graph Neural Networks. There is no direct path to any potential societal impact on the work performed.
    \item[] Guidelines:
    \begin{itemize}
        \item The answer \answerNA{} means that there is no societal impact of the work performed.
        \item If the authors answer \answerNA{} or \answerNo, they should explain why their work has no societal impact or why the paper does not address societal impact.
        \item Examples of negative societal impacts include potential malicious or unintended uses (e.g., disinformation, generating fake profiles, surveillance), fairness considerations (e.g., deployment of technologies that could make decisions that unfairly impact specific groups), privacy considerations, and security considerations.
        \item The conference expects that many papers will be foundational research and not tied to particular applications, let alone deployments. However, if there is a direct path to any negative applications, the authors should point it out. For example, it is legitimate to point out that an improvement in the quality of generative models could be used to generate Deepfakes for disinformation. On the other hand, it is not needed to point out that a generic algorithm for optimizing neural networks could enable people to train models that generate Deepfakes faster.
        \item The authors should consider possible harms that could arise when the technology is being used as intended and functioning correctly, harms that could arise when the technology is being used as intended but gives incorrect results, and harms following from (intentional or unintentional) misuse of the technology.
        \item If there are negative societal impacts, the authors could also discuss possible mitigation strategies (e.g., gated release of models, providing defenses in addition to attacks, mechanisms for monitoring misuse, mechanisms to monitor how a system learns from feedback over time, improving the efficiency and accessibility of ML).
    \end{itemize}
    
\item {\bf Safeguards}
    \item[] Question: Does the paper describe safeguards that have been put in place for responsible release of data or models that have a high risk for misuse (e.g., pre-trained language models, image generators, or scraped datasets)?
    \item[] Answer: \answerNA{} % Replace by \answerYes{}, \answerNo{}, or \answerNA{}.
    \item[] Justification: The paper poses no such risk as it does not contain any models or data that have high risk for misuse.
    \item[] Guidelines:
    \begin{itemize}
        \item The answer \answerNA{} means that the paper poses no such risks.
        \item Released models that have a high risk for misuse or dual-use should be released with necessary safeguards to allow for controlled use of the model, for example by requiring that users adhere to usage guidelines or restrictions to access the model or implementing safety filters. 
        \item Datasets that have been scraped from the Internet could pose safety risks. The authors should describe how they avoided releasing unsafe images.
        \item We recognize that providing effective safeguards is challenging, and many papers do not require this, but we encourage authors to take this into account and make a best faith effort.
    \end{itemize}

\item {\bf Licenses for existing assets}
    \item[] Question: Are the creators or original owners of assets (e.g., code, data, models), used in the paper, properly credited and are the license and terms of use explicitly mentioned and properly respected?
    \item[] Answer: \answerYes{} % Replace by \answerYes{}, \answerNo{}, or \answerNA{}.
    \item[] Justification: The datasets used are all properly cited in this work.
    \item[] Guidelines:
    \begin{itemize}
        \item The answer \answerNA{} means that the paper does not use existing assets.
        \item The authors should cite the original paper that produced the code package or dataset.
        \item The authors should state which version of the asset is used and, if possible, include a URL.
        \item The name of the license (e.g., CC-BY 4.0) should be included for each asset.
        \item For scraped data from a particular source (e.g., website), the copyright and terms of service of that source should be provided.
        \item If assets are released, the license, copyright information, and terms of use in the package should be provided. For popular datasets, \url{paperswithcode.com/datasets} has curated licenses for some datasets. Their licensing guide can help determine the license of a dataset.
        \item For existing datasets that are re-packaged, both the original license and the license of the derived asset (if it has changed) should be provided.
        \item If this information is not available online, the authors are encouraged to reach out to the asset's creators.
    \end{itemize}

\item {\bf New assets}
    \item[] Question: Are new assets introduced in the paper well documented and is the documentation provided alongside the assets?
    \item[] Answer: \answerYes{} % Replace by \answerYes{}, \answerNo{}, or \answerNA{}.
    \item[] Justification: All the code released contains sufficient documentation to be used.
    \item[] Guidelines:
    \begin{itemize}
        \item The answer \answerNA{} means that the paper does not release new assets.
        \item Researchers should communicate the details of the dataset\slash code\slash model as part of their submissions via structured templates. This includes details about training, license, limitations, etc. 
        \item The paper should discuss whether and how consent was obtained from people whose asset is used.
        \item At submission time, remember to anonymize your assets (if applicable). You can either create an anonymized URL or include an anonymized zip file.
    \end{itemize}

\item {\bf Crowdsourcing and research with human subjects}
    \item[] Question: For crowdsourcing experiments and research with human subjects, does the paper include the full text of instructions given to participants and screenshots, if applicable, as well as details about compensation (if any)? 
    \item[] Answer: \answerNA{} % Replace by \answerYes{}, \answerNo{}, or \answerNA{}.
    \item[] Justification: No crowdsourcing or human subjects is involved.
    \item[] Guidelines:
    \begin{itemize}
        \item The answer \answerNA{} means that the paper does not involve crowdsourcing nor research with human subjects.
        \item Including this information in the supplemental material is fine, but if the main contribution of the paper involves human subjects, then as much detail as possible should be included in the main paper. 
        \item According to the NeurIPS Code of Ethics, workers involved in data collection, curation, or other labor should be paid at least the minimum wage in the country of the data collector. 
    \end{itemize}

\item {\bf Institutional review board (IRB) approvals or equivalent for research with human subjects}
    \item[] Question: Does the paper describe potential risks incurred by study participants, whether such risks were disclosed to the subjects, and whether Institutional Review Board (IRB) approvals (or an equivalent approval/review based on the requirements of your country or institution) were obtained?
    \item[] Answer: \answerNA{} % Replace by \answerYes{}, \answerNo{}, or \answerNA{}.
    \item[] Justification: No crowdsourcing or human subjects is involved.
    \item[] Guidelines:
    \begin{itemize}
        \item The answer \answerNA{} means that the paper does not involve crowdsourcing nor research with human subjects.
        \item Depending on the country in which research is conducted, IRB approval (or equivalent) may be required for any human subjects research. If you obtained IRB approval, you should clearly state this in the paper. 
        \item We recognize that the procedures for this may vary significantly between institutions and locations, and we expect authors to adhere to the NeurIPS Code of Ethics and the guidelines for their institution. 
        \item For initial submissions, do not include any information that would break anonymity (if applicable), such as the institution conducting the review.
    \end{itemize}

\item {\bf Declaration of LLM usage}
    \item[] Question: Does the paper describe the usage of LLMs if it is an important, original, or non-standard component of the core methods in this research? Note that if the LLM is used only for writing, editing, or formatting purposes and does \emph{not} impact the core methodology, scientific rigor, or originality of the research, declaration is not required.
    %this research? 
    \item[] Answer: \answerYes{} % Replace by \answerYes{}, \answerNo{}, or \answerNA{}.
    \item[] Justification: The use of LLMs is discussed in Appendix \ref{app:llm}
    \item[] Guidelines:
    \begin{itemize}
        \item The answer \answerNA{} means that the core method development in this research does not involve LLMs as any important, original, or non-standard components.
        \item Please refer to our LLM policy in the NeurIPS handbook for what should or should not be described.
    \end{itemize}

\end{enumerate}